\documentclass[11pt]{article}

\usepackage{fullpage}
\usepackage[utf8]{inputenc} % allow utf-8 input
\usepackage[T1]{fontenc}    % use 8-bit T1 fonts
\usepackage{hyperref}       % hyperlinks
\usepackage{url}            % simple URL typesetting
\usepackage{booktabs}       % professional-quality tables
\usepackage{amsfonts}       % blackboard math symbols
\usepackage{nicefrac}       % compact symbols for 1/2, etc.
\usepackage{microtype}      % microtypography
\usepackage{graphicx}
\usepackage{caption}
\usepackage{subcaption}
\usepackage{xcolor}
\usepackage{soul}
\usepackage{doi}
\usepackage{authblk}
\usepackage{changepage}
\usepackage{amsmath}
\usepackage{amssymb}
\usepackage{enumerate}
\usepackage{epstopdf}
\usepackage[ruled,vlined]{algorithm2e}
\usepackage{algpseudocode}

\usepackage{mathtools}
\usepackage{setspace}
\usepackage{wrapfig}

\usepackage{amsthm}

\theoremstyle{plain}
\newtheorem{theorem}{Theorem}
\numberwithin{theorem}{section}
\newtheorem{proposition}{Proposition}
\numberwithin{proposition}{section}
\newtheorem{lemma}{Lemma}
\numberwithin{lemma}{section}

\numberwithin{corollary}{section}

\numberwithin{example}{section}

\theoremstyle{definition}
\newtheorem{definition}{Definition}
\numberwithin{definition}{section}
\newtheorem{remark}{Remark}
\numberwithin{remark}{section}

\numberwithin{equation}{section}

\providecommand{\keywords}[1]
{
  \small	
  \textbf{Keywords:} #1
}
\title{Bayesian neural networks with interpretable priors from Mercer kernels}

\date{January 31, 2026}

\begin{document}

\author[1]{Alex Alberts\thanks{Corresponding author   \href{mailto:albert31@purdue.edu}{\color{blue}{\texttt{albert31@purdue.edu}}}}}
\author[1]{Ilias Bilionis}
\affil[1]{School of Mechanical Engineering, Purdue University, West Lafayette, IN}

\maketitle

\begin{abstract}
    Quantifying the uncertainty in the output of a neural network is essential for deployment in scientific or engineering applications where decisions must be made under limited or noisy data. Bayesian neural networks (BNNs) provide a framework for this purpose by constructing a Bayesian posterior distribution over the network parameters. However, the prior, which is of key importance in any Bayesian setting, is rarely meaningful for BNNs. This is because the complexity of the input-to-output map of a BNN makes it difficult to understand how certain distributions enforce any interpretable constraint on the output space of the network. Gaussian processes (GPs), on the other hand, are often preferred in uncertainty quantification tasks due to their interpretability. The drawback is that GPs are limited to small datasets without advanced techniques, which often rely on the covariance kernel having a specific structure. To address these challenges, we introduce a new class of priors for BNNs, called Mercer priors, such that the resulting BNN has samples which approximate that of a specified GP. The method works by defining a prior directly over the network parameters from the Mercer representation of the covariance kernel, and does not rely on the network having a specific structure. In doing so, we can exploit the scalability of BNNs in a meaningful Bayesian way.
\end{abstract}


\onehalfspacing

\section{Introduction}
 
As neural networks begin to be employed in sensitive areas, such as complex engineering systems or healthcare~\cite{topol2019high, sarvamangala2022convolutional, kufel2023machine, hans2025smurf}, the need to reliably assess their predictions becomes critically important.
In these settings, it is oftentimes insufficient for a neural network to simply provide accurate point predictions, as the model must also quantify the uncertainty in the output to inform decisions.
For neural networks, this task can be formulated under the Bayesian paradigm by treating the network parameters as random variables and constructing a posterior distribution conditioned on the data, which leads to a Bayesian neural network (BNN)~\cite{mackay1992practical,neal2012bayesian}.
The complexity of the input-to-output map of a neural network makes it difficult to enforce any interpretable constraints on the output space of the BNN when assigning a specific form of the prior.
As a result, the default choice is simply to place an independent and identically distributed Gaussian prior over the individual network parameters.
Although this choice is certainly convenient, it is not always appropriate, and rarely translates to any meaningful constraints.
In contrast, Gaussian processes (GPs) offer a greater degree of interpretability for uncertainty quantification tasks.
This comes at a greater computational cost, however, as GPs do not scale well to large datasets without clever implementations that restrict the possible covariance structures.

Notably, a close relationship between BNNs and GPs is well-documented in the literature~\cite{lee2017deep, matthews2018gaussian, novak2018bayesian, dutordoir2021deep}.
For example, when the parameters of a BNN are drawn independently from a Gaussian distribution in the usual manner, the resulting distribution over functions converges to a GP in the infinite-width limit~\cite{neal2012bayesian}.
However, the form of the limiting GP depends entirely on the activation function of the network, as this determines the covariance kernel.
This means that specifying a desired GP prior would require identifying a tailor-made activation function whose induced kernel matches that of the GP.
While this has been shown to be possible for certain cases~\cite{williams1996computing, lee2017deep, yang2019output, ranftl2023physics}, a general procedure remains out of reach.
This observation motivates an alternative viewpoint: rather than fixing the activation function and sampling the parameters independently, is it instead possible to design the parameter distribution so that the BNN approximates a given GP?

In this work, we propose a solution by introducing a new class of priors for BNNs, which we call Mercer priors.
The key idea is to construct the prior over the BNN parameters directly from the Mercer representation of a target GP kernel, which then resembles the probability density function of Gaussian distribution.
In doing so, a covariance structure is placed over the BNN parameters such that the network resembles draws from the GP in function space.
This approach allows the BNN to inherit the interpretability of GP priors while retaining the scalability of neural networks.

We first motivate the technique and contextualize the Mercer prior within the available methods for sampling BNNs in Sec.~\ref{sec:background}.
The prior is introduced in Sec.~\ref{sec:method}, where we also derive a scalable sampling scheme and discuss possible covariance choices.
After defining the Mercer prior, we do a deep dive into its use in sampling BNNs which resemble Brownian motion draws in Sec.~\ref{sec:study}.
We compare the output statistics of the BNN to true Brownian motion, and provide numerical evidence that convergence is expected in the infinite-width limit.
We then present three application-driven examples in Sec.~\ref{sec:applications}: (i) hierarchical GP regression under heteroscedastic noise, (ii) time series prediction on data with a periodic structure, and (iii) a nonlinear elliptic PDE inverse problem with real-world data.
Together, these results illustrate how Mercer priors provide a principled and scalable way to endow BNNs with an interpretable structure, enabling them to be deployed in uncertainty quantification tasks that are currently out of reach for standard GP models.

\section{Background and setting}
\label{sec:background}

To motivate the method, we study the task of approximating a deterministic function $u : \Omega \to \mathbb{R}$, $\Omega \subset \mathbb
R^d$, with a BNN.
We will assume that $u \in L^2(\Omega)$.
We place a focus on the regression task, where we have noisy observational data $y \in \mathbb{R}^n$ appearing in the form $y = R(u) + \gamma$, where $R : L^2(\Omega) \to \mathbb{R}^n$ is the forward observation map, and $\gamma \sim \mathcal{N}(0,\Gamma)$ is additive noise.
The use of a BNN allows for uncertainty quantification about the reconstruction of $u$ from the data.
Fixing an activation function $h:\mathbb{R}\to\mathbb{R}$, we define a BNN $u_{\theta}$ with $L\geq 1$ hidden layers as the stochastic process
\begin{equation}
    \label{eqn:BNN}
    u_{\theta}(t) = \sum_{j=1}^{N_L} w^L_{1,j}h\left(z_{j}^L(t)\right), \quad z_i^s(t) = \sum_{j=1}^{N_{s-1}}w_{i,j}^{s-1}h\left(z_j^{s-1}(t)\right) + b_i^{s-1}, \quad s=2,\dots, L,
\end{equation}
where the input layer follows $z_i^1(t) = \sum_{j=1}^dw_{i,j}^0t_j + b_i^0$.
We let $N_s = \dim(z^s(t))$ denote the number of neurons in layer $s$.
The parameters of the BNN are the weights $w_{i,j}^s \in \mathbb{R}$ and the biases $b_i^s \in \mathbb{R}$, which are consolidated into the collection $\theta$.
Eq.~(\ref{eqn:BNN}) defines a stochastic process $u : \Omega \times \Theta \to \mathbb{R}$, where $\Theta$ is a measurable parameter space on which a prior distribution on the BNN parameters $\theta$ is assigned.
In the case $\dim(\Theta)<\infty$, i.e., for a finite-width network, we denote the prior over the BNN parameters as $p(\theta)$.
We are also interested in studying infinite-width networks, where it is no longer rigorous to write $p(\theta)$.
In this case, we denote the prior as some probability measure $\mu_0$ over the space $\Theta$.
Either way, the mapping $\theta \mapsto u_{\theta}$ is a random field taking values in $L^2(\Omega)$.

Training a BNN follows the Bayesian approach, where the latent function $u$ is modeled as the stochastic process in eq.~(\ref{eqn:BNN}).
Then, the BNN parameters are constrained to a posterior measure $\mu^y$ through Bayes's rule by conditioning on the data.
Under some light assumptions (and using our specific form of the data), Bayes's theorem states that $\mu^y$ is absolutely continuous with respect to $\mu_0$ with Radon-Nikodym derivative given by the likelihood:
$$
\frac{d\mu^y}{d\mu_0}(\theta) \propto \exp\left(-\frac{1}{2}\|\Gamma^{-1}(y-R(u_{\theta}))\|^2\right),
$$
which holds even when $\dim(\Theta) = \infty$~\cite{stuart2010inverse}.
In the case of a finite-width network, this reduces to the well-known formula $p(\theta|y) \propto p(y|\theta)p(\theta)$, with the likelihood being $p(y|\theta)\propto \exp\left(-\tfrac{1}{2}\|\Gamma^{-1}(y-R(u_{\theta}))\|^2\right)$.
Note that this form of the posterior arises under Gaussian measurement noise.
In the finite-dimensional case, as with training a BNN, non-Gaussian noise is easily handled by a straightforward application of Bayes's theorem.

\subsection{Wide neural networks and Gaussian processes}

One of the most critical components in any Bayesian scheme is the choice of the prior.
In a regression task, the prior should be chosen to reflect preexisting knowledge about $u$, e.g., regularity constraints.
The standard practice with BNNs is to place a prior over the parameters directly.
Due to the complexity of the BNN structure of eq.~(\ref{eqn:BNN}), it is difficult to interpret how any given prior distribution $p(\theta)$ enforces constraints on the output of the BNN in function space.
For this reason, in the majority of cases, $p(\theta)$ is simply chosen so that each parameter is a priori independent and normally distributed with a mean of zero and some tunable variance.
That is, we assume that each of the network weights follow $w^{s}_{i,j} \sim \mathcal{N}(0,\sigma^2_{w_s})$ with $\sigma_{w_s}^2 = \sigma^2/ N_s$ and the biases are distributed according to $b_i^{s} \sim \mathcal{N}(0,\sigma_{b_s}^2)$, for some $\sigma,\sigma_{b_s} > 0$ fixed.
The weight variance $\sigma^2_{w_s}$ is scaled by the number of neurons so that a limit theorem holds.

At initialization, it is known that the BNN converges asymptotically as the number of neurons in any layer approaches infinity to a zero-mean Gaussian process (GP) with a covariance kernel defined by the network's activation function~\cite{neal2012bayesian}.
The associated GP is known as the neural network Gaussian process (NNGP).
For a shallow network ($L=1$), the associated NNGP covariance is
$$
k(t,t') = \sigma_w^2\mathbb{E}[h(t^{\intercal}w_1^0+b_1^0)h(t'^{\intercal}w_1^0+ b_1^0)] + \sigma_b^2.
$$
When $L>1$, there is a recursive formula for $k$ with a similar structure.

In contrast to BNNs, the behavior of GPs in function spaces is very interpretable, making them popular in machine learning tasks~\cite{williams2006gaussian}.
This is attributed to their nice analytical properties, e.g., the covariance kernel of a GP completely determines how the sample paths behave.
The main drawback of GPs is the difficulty in scaling to large datasets.
Under the most basic setting, generating a sample path from a GP scales cubically with the number of data points.
Multiple methods have been proposed to combat this by clever manipulation of the covariance matrices, such as sparse GPs~\cite{snelson2005sparse, titsias2009variational, hensman2013gaussian}, KISS-GP~\cite{wilson2015kernel}, or the use of a separable kernel~\cite{bilionis2013multi}.
These methods typically require the data or evaluation points to be on a regular grid, which may not always be practical in certain applications, e.g., boundary layers.
Another scalable approach is deep kernel learning~\cite{wilson2016deep}, but GPs with deep kernels are known to oftentimes overfit~\cite{ober2021promises}.
Finally, one could represent the GP with the corresponding Karhunen-Lo\`eve expansion (KLE)~\cite{williams2006gaussian}.
For a centered GP the KLE takes the form
$$
u(\cdot) = \sum_{n=1}^{\infty}\sqrt{\lambda_n}\xi_n\phi_n(\cdot),
$$
where $\lambda_n$ and $\phi_n$ are the eigenvalues and eigenfunctions of the kernel integral operator $Su = \int k(\cdot,t)u(t)dt$, and $\xi_n\overset{\text{i.i.d.}}{\sim}\mathcal{N}(0,1)$.
The expansion is truncated to a finite number of terms $M$ in practice.
When using the KLE, the computational complexity is transferred from inverting a large matrix to constructing the basis functions.
However, since the eigenfunctions $\phi_n$ are global and fixed, performance may suffer when the data is non-uniformly spaced or sparse, unless $M$ is large.
Also, the KLE suffers from the curse of dimensionality, as the number of eigenfunctions needed to capture most of the variance grows exponentially with the dimension.
A comprehensive review of these ideas can be found in~\cite{liu2020gaussian}.

\subsection{Related methods}

Looking to gain both the expressivity of GPs and scalability of BNNs, several methods have been proposed that exploit the relationship between the two.
This can be done using the correspondence between a BNN and the associated NNGP, so that the problem of prior selection for a BNN is transferred to the easier problem of prior selection for GPs.
One class of methods selects the prior in order to minimize some objective between the BNN and the target GP~\cite{flam2017mapping,hafner2020noise,sun2019functional}.
The drawback is that the user is constrained to a particular algorithm for posterior inference (like variational inference).
These methods also typically need to see some of the training data upfront in order to construct the prior, e.g.~\cite{meng2022learning}.
Other methods rely on changing the architecture of the BNN to ensure convergence to the GP in an appropriate sense~\cite{pearce2020expressive, albert2020gaussian, ranftl2023physics}.
This of course works on a case-by-case basis.

Perhaps the most direct method is to place a distribution on the parameters so that the BNN statistics begin to match those of the GP.
This allows for any architecture or sampling algorithm to be used.
The best known attempt to do so uses the ridgelet prior of~\cite{matsubara2021ridgelet}.
\textbf{}This technique makes use of the fact that a feedforward neural network is equivalent to a quadrature of a ridgelet and dual ridgelet transformation~\cite{candes1998ridgelets, sonoda2017neural, sonoda2018global}.
While this method guarantees convergence to the desired GP, it is potentially bottlenecked by the need to invert covariance matrices layer by layer in order to evaluate the prior density.
As noted by the authors, this scales cubically in the layer width $N_{s}$, which may limit the method to smaller networks.
It is also worth noting that the ridgelet prior suffers from the curse of dimensionality.
Because computing the ridgelet prior involves numerical approximation of the ridgelet and dual ridgelet transformations with quadrature rules, this does not scale well to high-dimensional input spaces.
This can be explicitly seen in the error bound given in~\cite[Theorem 1]{matsubara2021ridgelet}, where an $O(D^{-1/d})$ term appears, with $D$ being the number of quadrature points used to approximate the ridgelet transformation.

The ridgelet prior and the Mercer prior we present in this work are attempting to solve the same problem.
Namely that we want to sample the parameters of a BNN such that the resulting function resembles draws from a predetermined GP.
The difference then comes down to how the prior over the network parameters is characterized.
In this sense, one can view the two as different finite-dimensional approximations of the same GP.
However, computing the Mercer prior scales \emph{linearly} in the number of network parameters, which allows for the potential use of larger networks.
While the Mercer prior does contain two integrals which need to be computed, we avoid the curse of dimensionality by deriving a sampling scheme which uses an unbiased estimate of the prior.
This is built with Monte Carlo estimates, which in principle guarantees convergence even for a batch size of $1$ (this is equivalent to setting $D=1$ in the ridgelet prior).

\section{The Mercer prior}
\label{sec:method}

Our method relies on the relationship between GPs and Gaussian measures~\cite{rajput1972gaussian}.
That is, we build the approximation by first thinking of the sample paths of a desired GP as equivalently being drawn from an induced Gaussian measure on a function space.
We then build a finite-dimensional approximation of this measure using neural networks.
The references~\cite{kuo2006gaussian} and~\cite{bogachev1998gaussian} provide a background on the theory of Gaussian measures in infinite-dimensional spaces, and~\cite{williams2006gaussian} serves as an introduction to the theory of Gaussian processes.

Start by defining a GP on $\Omega \subset \mathbb{R}^d$ that we wish to emulate $u \sim \mathcal{GP}(0,k)$, where $k : \Omega \times \Omega \to \mathbb{R}$ is the covariance kernel.
Note that we are choosing to work with a centered GP to simplify some of the details.
A non-zero mean function could be used instead, provided that it satisfies certain regularity constraints.
The resulting formulae would then need to be adjusted.
Alternatively, one could shift the space so that the GP becomes centered.
Further, for the covariance kernel, we enforce the restriction $\int_{\Omega}k(t,t)dt < \infty$.
This is so that we may use the following theorem, which relates the GP to an induced Gaussian measure on $L^2(\Omega)$.
A positive-definite kernel statisfying these properties is sometimes called a Mercer kernel.

\begin{theorem}[Theorem 2~\cite{rajput1972gaussian}]
    \label{thm:GPGM}
    Let $u \sim \mathcal{GP}(m,k)$ be a measurable Gaussian process.
    Then, the sample paths $u \in L^2(\Omega)$ with probability $1$ if and only if
    $$
    \int_{\Omega}m^2(t)dt < \infty, \quad \int_{\Omega}k(t,t) dt < \infty.
    $$
    In this case, $u$ induces the Gaussian measure $\mathcal{N}(m,S)$ on $L^2(\Omega)$ with the covariance operator being $(Sv)(\cdot) = \int_{\Omega} k(\cdot,t)v(t)dt$, for $v \in L^2(\Omega)$.
\end{theorem}

The goal is to sample BNNs from the Gaussian measure $\mathcal{N}(0,S)$ corresponding to the GP with which we started.
That is, we seek to derive a distribution over the BNN parameters $\theta$ such that $u_{\theta} \sim \mathcal{N}(0,S)$.
After doing so, the BNN can serve as a replacement of the GP or Gaussian measure in applications.
To identify such a distribution $p(\theta)$, we apply the principles of information field theory~\cite{ensslin2009information}.
Ultimately, this is inspired by the free theory case of quantum field theory~\cite{glimm2012quantum} which deals with Gaussian measures.
Following ideas from quantum field theory, we express the Gaussian measure using its formal Lebesgue density via Feynman path integrals~\cite{feynman1979path}.

The formal Lebesgue density is derived by mimicking the probability density function of a multivariate Gaussian in the finite-dimensional setting.
That is, adopting the path integral formalism, we write the Gaussian measure $\mu = \mathcal{N}(0,S)$ as
\begin{equation}
    \label{eqn:WM}
    \mu(du) = \frac{1}{\mathcal{Z}} \exp\left(-\frac{1}{2}\left\langle u, S^{-1}u\right\rangle\right)\mathcal{D}u,
\end{equation}
where $\mathcal{Z}$ is the normalization constant of the measure defined in terms of a path integral
$$
\mathcal{Z} = \int_{L^2(\Omega)} \exp\left(-\frac{1}{2}\langle u, S^{-1}u\rangle\right)\mathcal{D}u.
$$

The form of eq.~(\ref{eqn:WM}) cannot immediately be taken to be well-defined, as the Lebesgue measure, which $\mathcal{D}u$ is representing, does not exist on infinite-dimensional Banach spaces.
However, the formalism is still useful for deriving techniques involving Gaussian measures, for which the form can be given a rigorous treatment using limiting procedures.
Such formal manipulations of path integrals can also be treated rigorously using Wick expansions~\cite{nguyen2016perturbative}.
Reconciling these issues is beyond the scope of this work.
Rather, we will work with a finite-dimensional representation of eq.~(\ref{eqn:WM}), where $\mathcal{D}u$ becomes the usual Lebesgue measure, which of course is well defined.

Now, let $u_\theta$ be a neural network with parameters $\theta \in \mathbb{R}^m$ and activation function sufficiently regular so that $u_\theta \in \mathrm{dom}(S^{-1}) $ for any choice $\theta$.
Replacing $u$ by this parameterization and substituting into eq.~(\ref{eqn:WM}), we obtain the following finite-dimensional measure
\begin{equation*}
    \hat{\mu}(d\theta) \propto \exp\left(-\frac{1}{2}\langle u_\theta, S^{-1}u_\theta\rangle\right)d\theta,
\end{equation*}
where $d\theta$ is the Lebesgue measure in $m$ dimensions.
Hence, we are justified in writing a proper probability density function over the network parameters
\begin{equation}
    \label{eqn:NNWM}
    p(\theta) \propto \exp\left(-\frac{1}{2}\langle u_\theta, S^{-1}u_\theta\rangle\right).
\end{equation}

We then numerically characterize eq.~(\ref{eqn:NNWM}) using traditional Monte Carlo sampling schemes in order sample BNNs.
Each sample of $\theta$ provides a new neural network, and we will see that the ensemble statistics approximately follow that of the original Gaussian measure $\mu$.
This offers a significant reduction in computational complexity when the BNN sampled from eq.~(\ref{eqn:NNWM}) replaces a GP for the same task.
As mentioned previously, the complexity of a GP scales with the square of the evaluation mesh.
On the other hand, the BNN prior of eq.~(\ref{eqn:NNWM}) grows with $\theta$.
After sampling, the BNN can be evaluated on a domain with a near-arbitrarily fine mesh, which allows for easy super-resolution.

The catch is that in order to generate a sample of $\theta$ to draw a new BNN, we must evaluate the $L^2(\Omega)$ inner product found in the density.
Also, the inverse of the covariance operator must be dealt with.
Both are defined as integrals, which on first look would be computationally demanding.

\begin{remark}
    We introduce some notation regarding covariance operators.
    Recall the definition of the covariance operator $S$ appearing in Theorem~\ref{thm:GPGM}.
    We will denote the kernel of the inverse covariance operator $S^{-1}$, also known as the precision operator, by $k^{-1}$.
    That is, $S^{-1}$ is defined by
    $$
    (S^{-1}v)(\cdot) = \int_{\Omega} k^{-1}(\cdot,t)v(t)dt, \quad v \in L^2(\Omega),
    $$
    such that $S^{-1}(Sv) = v$.
\end{remark}

We work our way around the need to evaluate the integrals by deriving a sampling scheme which uses stochastic gradient Langevin dynamics (SGLD)~\cite{welling2011bayesian}.
The beauty of SGLD is that samples from $p(\theta)$ can be generated using only unbiased estimates of $\nabla_{\theta} \log p(\theta) = -1/2\nabla_{\theta}\langle u_{\theta},S^{-1}u_{\theta}\rangle$.
This allows us to approximate the various integrals required using sampling averages and SGLD guarantees convergence to the correct probability distribution~\cite{teh2016consistency}.

The integrals are relatively easy, as unbiased estimates can be found in the usual way with importance sampling.
Representing the inverse covariance $k^{-1}$ is a bit more involved, but we derive a representation which is able to exploit the advantages of SGLD.
To do so, we rely on Mercer's theorem to write $k^{-1}$ with a spectral representation of $S^{-1}$.
The assumptions we have placed on $k$, namely that it is a Mercer kernel, imply that $S$ is a self-adjoint, positive operator (this is obvious since $S$ is the covariance operator of a Gaussian measure).
Therefore, $S$ has the spectral decomposition
$$
(Su)(\cdot) = \sum_{n=1}^{\infty} \lambda_n \langle u,\phi_n\rangle \phi_n(\cdot), \quad u \in L^2(\Omega),
$$
where $\{\lambda_n\}_{n\in\mathbb{N}}$ are the eigenvalues of $S$, with $\{\phi_n\}_{n\in\mathbb{N}}$ being the corresponding eigenfunctions.
That is, for each pair $(\lambda_n,\phi_n)$ it follows that
$$
S\phi_n = \lambda_n\phi_n.
$$
Further, each $\lambda_n \geq 0$ and $\lambda_n \to 0$.

We then rely on Mercer's theorem to express the covariance kernel $k$ using the eigenvalues and eigenfunctions of $S$.
\begin{theorem}[Mercer's theorem~\cite{steinwart2008support}]
    \label{thm:Mercer}
    
    Let $k: \Omega \times \Omega \to \mathbb{R}$ be a continuous, positive-definite kernel, and $S$ be the operator on $L^2(\Omega)$ defined by $u \mapsto \int_{\Omega} k(\cdot,t)u(t)dt.$
    Denote the eigenvalues and eigenfunctions of $S$ as $\{\lambda_n\}_{n\in\mathbb{N}}$ and $\{\phi_n\}_{n\in\mathbb{N}}$, respectively.
    Then $k$ can be expressed as
    $$
    k(s,t) = \sum_{n=1}^{\infty} \lambda_n\phi_n(s)\phi_n(t), \quad s,t\in\Omega,
    $$
    where the convergence is absolute and uniform.
\end{theorem}

The inverse covariance will have the same eigenfunctions, with the difference being that the eigenvalues are inverted, i.e., $S^{-1}\phi_n = \lambda_n^{-1}\phi_n.$
So, by Mercer's theorem, we may represent the inverse covariance kernel as
\begin{equation}
    \label{eqn:mercer}
    k^{-1}(s,t) = \sum_{n=1}^{\infty} \lambda_n^{-1}\phi_n(s)\phi_n(t), \quad s,t\in\Omega.
\end{equation}
For sampling the BNN, we will use the Mercer representation of $k^{-1}$ in eq.~(\ref{eqn:mercer}).
We will also use importance sampling to write $k^{-1}$ as an expectation, which we then take a sampling average of to build an unbiased estimate for SGLD steps.
As the form of our BNN prior is defined via a Mercer kernel, we refer to it as a Mercer prior.

To this end, let
\begin{equation}
    \label{eqn:prior}
    p(\theta) \overset{\Delta}{=} \exp\left(-\frac{1}{2}\langle u_{\theta}, S^{-1}u_{\theta}\rangle\right)
\end{equation}
be the parameter distribution, where we use $\overset{\Delta}{=}$ to represent equivalence up to the normalization constant.
Since SGLD draws a sample from $p(\theta)$ using $\nabla_{\theta}\log p(\theta)$, the normalization constant vanishes in the derivative and can be dropped.
Putting everything together, we obtain the following unbiased estimator of $\log p(\theta)$.
\begin{proposition}
    \label{prop:prop}
    Let $k$ be a symmetric, positive-definite kernel on $\Omega \times \Omega$, $|\Omega| < \infty$, satisfying $\int_{\Omega} k(t,t)dt < \infty$, and $S$ and $S^{-1}$ be the integral operators with kernels $k$ and $k^{-1}$, respectively.
    Let $u_{\theta} \in \mathrm{dom}(S^{-1})$ be parameterized by $\theta \in \mathbb{R}^m$. 
    Then, a valid unbiased estimator for $\log p(\theta)$ as given by eq.~(\ref{eqn:prior}) is
    \begin{equation}
        \label{eqn:BNNdist}
        \mathcal{S}_{N,M_1,M_2} = -\frac{1}{2} \frac{|\Omega|^2}{NM_1M_2}\sum_{\alpha = 1}^N \frac{1}{\lambda_{n_{\alpha}} p({n_{\alpha}})}\left\{\sum_{\beta = 1}^{M_1} u_{\theta}(t_{\beta})\phi_{n_{\alpha}}(t_{\beta})\right\}\left\{\sum_{\gamma = 1}^{M_2} u_{\theta}(t_{\gamma})\phi_{n_{\alpha}}(t_{\gamma})\right\},
    \end{equation}
    where $t_{\beta} \overset{\mathrm{i.i.d.}}{\sim} U(\Omega)$, $t_{\gamma} \overset{\mathrm{i.i.d.}}{\sim} U(\Omega)$, for $\beta = 1,\dots, M_1$ and $\gamma = 1,\dots, M_2$, and $n_{\alpha}$, $\alpha = 1,\dots,N$, be i.i.d. draws from a discrete random variable with probability mass function $p(n)$, supported on $\mathbb{N}$.
    Here, $\lambda_{n_{\alpha}}$ and $\phi_{n_{\alpha}}$ represent the $n_{\alpha}$-th eigenvalue and corresponding eigenfunction of the operator $S$, respectively.
\end{proposition}
\begin{proof}
    Explicitly, we have
    $$
    \log p(\theta) = -\frac{1}{2} \int_{\Omega} u(s) \int_{\Omega} k^{-1}(s,t)u(t)dt \:ds + \mathrm{const}.
    $$
    The assumptions on $k$ imply that $S$ is a trace-class operator, hence it is also a Hilbert-Schmidt operator.
    Therefore, the spectral theorem applies and the eigenvalues $\lambda_{n}$ and eigenfunctions $\phi_n$ exist.
    Also, each $\lambda_n$ is real and positive since $k$ is positive-definite.

    We now invoke Mercer's theorem~\ref{thm:Mercer}, which allows us to express $k^{-1}$ in terms of the eigenvalues and eigenfunctions of $S$,
    $$
    k^{-1}(s,t) = \sum_{n=1}^{\infty}\frac{1}{\lambda_n}\phi_n(s)\phi_n(t).
    $$
    Returning to the form of the prior, and inserting the Mercer representation, we find
    \begin{align}
        \label{eqn:series}
        \log p(\theta) &\overset{\Delta}{=} -\frac{1}{2}\int_{\Omega}u(s)\int_{\Omega}\sum_{n=1}^{\infty}\frac{1}{\lambda_n}\phi_n(s)\phi_n(t)u(t)dt\:ds \nonumber  \\ 
        &= -\frac{1}{2}\sum_{n=1}^{\infty}\frac{1}{\lambda_n}\int_{\Omega} u(s)\phi_n(s)ds\int_{\Omega}u(t)\phi_n(t)dt,
    \end{align}
    where interchanging the summation and the integrals is valid since Mercer's theorem guarantees uniform convergence.
    
    For the sampling averages, we employ importance sampling to write each piece of eq.~(\ref{eqn:series}) as an expectation.
    Firstly, if we let $T \sim U(\Omega)$, then we have
    \begin{align*}
        \int_{\Omega} u_{\theta}(t)\phi_n(t)dx & = \int_{\Omega} \frac{|\Omega|}{|\Omega|}u_{\theta}(t)\phi_n(t)dt\\
        &= |\Omega|\mathbb{E}_{T\sim U(\Omega)}[u_\theta(T)\phi_n(T)] \\
        &\approx \frac{|\Omega|}{M_1}\sum_{\beta=1}^{M_1}u_{\theta}(t_{\beta})\phi_n(t_\beta),
    \end{align*}
    where $t_{\beta}$, $\beta = 1,\dots,M_1$ are i.i.d. draws from $U(\Omega)$. 
    This applies for both integrals.
    
    As for the outer summation, taking $p(n)$ to be the probability mass function of any discrete random variable supported on $\mathbb{N}$, we may write
    \begin{align*}
        \sum_{n=1}^{\infty}\frac{1}{\lambda_n}\int_{\Omega} u(s)\phi_n(s)ds\int_{\Omega}u(t)\phi_n(t)dt &= \sum_{n=1}^{\infty}\frac{p(n)}{\lambda_n p(n)}\langle u, \phi_n\rangle\langle u,\phi_n\rangle \\
        &= \mathbb{E}\left[\frac{1}{\lambda_n p(n)}\langle u, \phi_n\rangle\langle u,\phi_n\rangle\right] \\
        &\approx \frac{1}{N}\sum_{\alpha=1}^N\frac{1}{\lambda_{n_\alpha}p(n_{\alpha})}\langle u, \phi_{n_\alpha}\rangle\langle u,\phi_{n_\alpha}\rangle,
    \end{align*}
    where $n_{\alpha}$ with $\alpha = 1,\dots,N$ are i.i.d. draws from $p(n)$, and the expectation is taken over the chosen discrete random variable.
    Inserting the summations into $\log p(\theta)$ and taking care to track the sampling-average constants reveals the desired estimate.
\end{proof}

\begin{remark}
    Observe that in proposition~\ref{prop:prop} we separate the two inner products $\langle u_{\theta},\phi_{n_\alpha}\rangle$, which are essentially the same, into two separate expectations.
    One may be tempted to instead work with one estimate which is then squared.
    However, this is incorrect.
    It is important to separate the inner products into individual estimates which are computed using separate minibatches.
    This is because in general for random variables $X$ and $Y$, $\mathbb{E}[X]\mathbb{E}[Y] = \mathbb{E}[XY]$ only if $X$ and $Y$ are independent.
\end{remark}

Proposition~\ref{prop:prop} provides an unbiased estimator of the Mercer prior as given by eq.~(\ref{eqn:prior}).
The estimate exploits the available spectral representation of the covariance operator $S$ for the measure we wish to draw the BNN from.
In this way, we avoid ever needing to compute the integrals explicitly, identify the inverse covariance kernel analytically, or invert a large covariance matrix as is typical when sampling from a GP.
At any given step, to compute the estimate, we simply need to minibatch points in the domain in order to approximate the integrals and select random eigenvalues/eigenfunctions associated with $S$.
The properties of SGLD guarantee convergence to the correct probability distribution even for all batch sizes equal to one, showing the potential for massive parallelization.
For convenience, we summarize a sampling scheme for the Mercer prior in Algorithm~\ref{alg:mercer}.

In proposition~\ref{prop:prop}, we have chosen to minibatch points in $\Omega$ uniformly, which does not allow the domain to be infinite.
This choice was made in order to simplify the method and it is not a strict requirement.
If one wished to expand the domain to $\mathbb{R}^d$, importance sampling can be used to turn the integrals $\langle u_{\theta}, \phi_n\rangle$ into expectations.
Specifically, if we instead sample $t_{\beta}$ and $t_{\gamma}$ from distributions supported on $\mathbb{R}^d$, $q(t)$ and $w(t)$, respectively, we approximate each inner product using
$$
\langle u_{\theta}, \phi_n\rangle \approx \sum_{\beta = 1}^{M_1} \frac{u_{\theta}(t_{\beta})\phi_n(t_{\beta})}{q(t_{\beta})}, \quad 
$$
where the other inner product takes on the same form.
However, we see no reason to use different importance sampling distributions for the inner products, i.e., we suggest letting $q = w$.

\begin{algorithm}[t]
\caption{Mercer prior sampler}
\label{alg:mercer}
\begin{algorithmic}[1]
\Require 
initial parameters $\theta_0$, total iterations $T$, spectral indices $N$, spatial minibatches $M_1$, $M_2$, initial learning rate $\varepsilon_0$\
\Ensure 
Parameter samples $\{\theta_j\}_{j=0}^T$

\State initialize parameters $\theta_j \gets \theta_0$ and learning rate $\varepsilon_j \gets \varepsilon_0$\;

\State \For{$j = 0$ to $T-1$}
{ 
    sample minibatch 
    $\mathcal{B}_1 = \{t_\beta\}_{\beta=1}^{M_1} \sim U(\Omega)$\;
    sample minibatch 
    $\mathcal{B}_2 = \{t_\gamma\}_{\gamma=1}^{M_2} \sim U(\Omega)$\;
    sample spectral index minibatch 
    $\mathcal{I} = \{n_\alpha\}_{\alpha=1}^{N} \sim p(n)$\;
    sample SGLD noise $\eta_j \sim \mathcal{N}(0, \varepsilon_j I)$\;

    compute inner product estimators:
    \[
        (v_1)_\alpha
        \gets
        \sum_{t \in \mathcal{B}_1}
        u_{\theta_j}(x)\,\phi_{n_\alpha}(t),
        \qquad n_\alpha \in \mathcal{I}
    \]
    \[
        (v_2)_\alpha
        \gets
        \sum_{t \in \mathcal{B}_2}
        u_{\theta_j}(t)\,\phi_{n_\alpha}(t),
        \qquad n_\alpha \in \mathcal{I}\;
    \]

    form unbiased Monte Carlo estimator:
    \[
        \mathcal{S}_{N,M_1,M_2}(\theta_j)
        \gets
        -\frac{1}{2}
        \frac{|\Omega|^2}{N M_1 M_2}
        \sum_{n_\alpha \in \mathcal{I}}
        \lambda_{n_\alpha}^{-1}
        p(n_\alpha)^{-1}
        (v_1)_\alpha (v_2)_\alpha\;
    \]

    SGLD update:
    \[
        \theta_{j+1}
        \gets
        \theta_j
        + \frac{1}{2}\varepsilon_j
        \nabla_\theta \mathcal{S}_{N,M_1,M_2}(\theta_j)
        + \eta_j\;
    \]

    update learning rate $\varepsilon_{j+1}$\;
}
\State \Return parameter samples $\{\theta_j\}_{j=0}^T$
\end{algorithmic}
\end{algorithm}

\subsection{A cruder approximation}
 
We have written proposition~\ref{prop:prop} so that any discrete random variable on $\mathbb{N}$ may be used to sample the eigenvalues and eigenfunctions.
Intuitively, the conventional wisdom suggests to select the distribution so that it decays at the same rate as the eigenvalues.
As an example consider the eigenvalues of Brownian motion which decay like $O(n^{-2})$.
This leads us to believe that the zeta distribution with parameter $2$, i.e.,
$$
p(n) = \frac{6}{\pi^2n^2},
$$
would be an appropriate choice.
For eigenvalues which decay exponentially, one could set $p(n)$ to be the geometric distribution, and so on.
Of course, this discussion is merely formal.

In practice, it is unnecessary to keep all possible eigenvalues in the expansion.
The first reason is that if we allow $p(n)$ to be supported on the entirety of $\mathbb{N}$, there is a non-zero chance that a sufficiently small eigenvalue that degrades the numerical computation could be selected, i.e., the ratio $1/\lambda_n$ could blow up.
Also, there is a certain cut-off point where including additional terms in the expansion does not meaningfully contribute to the approximation.
To see this, observe that the Mercer approximation is expressing the kernel in terms of orthogonal projections onto subspaces defined by the eigenfunctions.
The eigenfunctions associated with larger eigenvalues are more dominant and contribute more of the required information contained in the kernel.
So, if we order the eigenvalues monotonically, i.e., $\lambda_{n+1}\leq\lambda_n$, there is some point where we will have captured enough of this information to build a robust approximation.
This is the same idea as what is found in principle component analysis.

For these reasons, in our applications throughout this work we truncate the expansion coming from the Mercer representation at some finite point $K$.
We show how the performance of the method adjusts for different values of $K$ under the same neural network structure.
Then, we choose for $p(n)$ the discrete uniform distribution, i.e., each integer from $1$ to $K$ is assigned equal probability.
When doing so the unbiased estimate becomes
\begin{equation}
    \label{eqn:sampler}
    \mathcal{S}_{N,M_1,M_2} = -\frac{1}{2} \frac{K|\Omega|^2}{NM_1M_2}\sum_{\alpha = 1}^N \frac{1}{\lambda_{n_{\alpha}}}\left\{\sum_{\beta = 1}^{M_1} u_{\theta}(t_{\beta})\phi_{n_{\alpha}}(t_{\beta})\right\}\left\{\sum_{\gamma = 1}^{M_2} u_{\theta}(t_{\gamma})\phi_{n_{\alpha}}(t_{\gamma})\right\},
\end{equation}
where all things remain equal as with proposition~\ref{prop:prop} except that each $n_{\alpha}$ is sampled from the discrete uniform distribution up to the integer $N$.
This may actually be a poor choice, as the larger eigenvalues are associated with more informative projections, so when building the approximation it is better if  $p(n)$ is decreasing with $n$.
We show that the method works well even under this suboptimal choice.

\subsection{Specifying Gaussian process covariance functions from the ground up}
\label{subsec:kernels}

Observe that the method builds sampling approximations of GPs using only the eigenvalues and eigenfunctions of the covariance kernel.
This means that the covariance kernel does not need to be explicitly written.
The benefit of this is that we may specify the desired GP completely through the eigenvalues and eigenfunctions, which allows us to construct BNNs for specialized tasks.
We describe two possible approaches for constructing a desired kernel.

First, note that given any set of orthonormal functions $\{\phi_n\}_{n=1}^{\infty} \subset L^2(\Omega)$, and any sequence $\{\lambda_n\}_{n=1}^{\infty}$ satisfying $\lambda_n >0$ and $\sum_{n=1}^{\infty}\lambda_n^2 < \infty$, we have that
\begin{equation}
    \label{eqn:GPkernel}
    k(s,t) = \sum_{n=1}^{\infty}\lambda_n\phi_n(s)\phi_n(t)
\end{equation}
defines a valid covariance kernel for a GP.
Note that this also holds if the series is truncated at some finite number of terms.
In this way, we may start by choosing the eigenvalues and eigenfunctions so that the GP follows some desired behavior, e.g., smoothness, which is then transferred into the behavior of the BNN.
The importance of this is that for GP kernels commonly used in machine learning tasks, the eigenvalues and eigenfunctions are currently not known or not easily computed.
As an example, the eigenvalues and eigenfunctions of the squared exponential kernel $k(s,t) = \exp(-\|s-t\|^2/\ell^2)$, which is perhaps the most commonly used kernel, are only available under the Gaussian measure $\mathcal{N}(0,\sigma^2)$, for some variance $\sigma^2> 0$~\cite{zhu1997gaussian}.
The eigenfunctions are then defined via transformations of the Hermite polynomials.
In numerical experiments not covered in this work, we found that we could not include enough of these eigenfunctions in the expansion to build a good approximation of the GP.
This was because after roughly $100$ terms, the computations became very unstable, mostly attributed to the fact that the normalization constants became exceptionally large.

In GP regression, the behavior of the prior covariance kernel greatly impacts the resulting predictions.
Therefore, it is important to understand how predetermined eigenvalues and eigenfunctions impact the resulting sample path behavior.
One can understand the underlying workings of this through the theory of reproducing kernel Hilbert spaces (RKHSs) and the connections to GP regression~\cite{kanagawa2018gaussian}.
RKHSs are defined in the following way:
\begin{definition}[Reproducing kernel Hilbert space~\cite{kanagawa2018gaussian}]
    Let $k$ be a symmetric, positive-definite mapping on $\Omega \times \Omega$.
    The reproducing kernel Hilbert space $H_k$ with inner product $\langle \cdot, \cdot\rangle_{H_k}$ is the Hilbert space on $\Omega$ satisfying
    \begin{enumerate}[(i)]
        \item For all fixed $t \in \Omega$, $k(\cdot,t) \in H_k$.
        \item For all fixed $t \in \Omega$, $u(t) = \langle u, k(\cdot,t)\rangle_{H_k}$ for each $u \in H_k$.
    \end{enumerate}
\end{definition}
One can show that each RKHS $H_k$ is uniquely determined by the kernel $k$ and that for each symmetric, positive-definite mapping $k$, there is a unique RKHS induced by this map~\cite{aronszajn1950theory}.
In this way we observe that when we define a kernel using the form of eq.~(\ref{eqn:GPkernel}), we are implicitly defining a unique RKHS of interest.
Then, one can show that the posterior mean function of GP regression with prior $\mathcal{GP}(0,k)$ will live in the space $H_k$~\cite[Proposition 3.6]{kanagawa2018gaussian}.
Meaning that one can carefully design a kernel so that it matches the assumed behavior of latent target function.
Also of note is that if we work with a finite number of terms in eq.~(\ref{eqn:GPkernel}) so that $\dim(H_k)<\infty$, then there exists a version of $\mathcal{GP}(0,k)$ with sample paths almost surely belonging to $H_k$~\cite[Corollary 4.10]{kanagawa2018gaussian}.
The same does not hold if $H_k$ is infinite-dimensional.

The second method involves damping an existing covariance with a Mercer representation.
One specific target application we have in mind for the Mercer prior is in infinite-dimensional Bayesian inverse problems, where one would look to replace the usual prior Gaussian measure with the corresponding BNN for scalability.
In such inverse problems, one often works with a `Laplace-like' covariance, see~\cite{stuart2010inverse}.
This is first described by the baseline Gaussian measure $\mathcal{N}(0,\Delta^{-1})$, where $\Delta$ is the Laplacian with Dirichlet boundary conditions.
This measure is famously known as the Brownian bridge measure.
On $[0,1]$, the kernel of the inverse Laplacian operator $\Delta^{-1}$ is simply $k(s,t) = \min\{s,t\} - st$.
On the $d$-dimensional unit cube, the kernel is best expressed in terms of its eigenvalues and eigenfunctions, given by
\begin{align*}
    \lambda_{n_1,\dots,n_d} &= \frac{1}{(n_1\cdots n_d)^2\pi^2} \\
    \phi_{n_1,\dots,n_d}(t) &= 2^{d/2}\sin(n_1\pi t)\cdots \sin(n_d\pi t),
\end{align*}
respectively~\cite[Sec. 8.2.2-16]{polyanin2001handbook}.
Hence, by Mercer representation the kernel can be written explicitly as
$$
k(s,t) = 2^d \sum_{n_1=1}^{\infty}\cdots \sum_{n_d=1}^{\infty} \frac{\sin(n_1\pi s)\cdots \sin(n_d\pi s)\sin(n_1\pi t)\cdots \sin(n_d\pi t)}{(n_1\cdots n_d)^2\pi^2}.
$$

This prior produces nondifferentiable sample paths, which is often undesirable when solving inverse problems.
So, we typically work with a damped version of this prior which enforces further smoothness by raising the inverse Laplacian to a power, which is defined as follows.
If $S$ is an operator for which the spectral theorem holds (like $\Delta^{-1}$), we define $S^{\alpha}$, for some $\alpha \in \mathbb{R}$ to be the operator
$$
(S^{\alpha}u)(x) = \sum_{n=1}^{\infty} \lambda_n^{\alpha}\langle u, \phi_n\rangle\phi_n.
$$
To understand how the commonly used prior $\mathcal{N}(0,\Delta^{-\alpha})$ behaves, we state the following two lemmas.
Before doing so, we define the following space, which appears in the second lemma,
$$
\mathcal{H}^s = \left\{ u : \sum_{n=1}^{\infty} \lambda_n^s |\langle u,\phi_n\rangle|^2 < \infty \right\},
$$
and note that $\mathcal{H}^s$ is a subspace of the usual square-integrable Sobolev space of order $s$.
\begin{lemma}[Lemma 6.25~\cite{stuart2010inverse}]
    Let $\Delta$ be the Laplacian on $\Omega$ with Dirichlet boundary conditions, and let $\mu = \mathcal{N}(0, \Delta^{-\alpha})$, for $\alpha > d/2$.
    Then a sample path $u$ of $\mu$ is a.s. $s$-H\"older for any exponent $s < \min(1, \alpha - d/2)$.
\end{lemma}

\begin{lemma}[Lemma 6.27~\cite{stuart2010inverse}]
    \label{lemma:sobolev}
    Let $\Delta$ be the Laplacian on $\Omega$ with Dirichlet boundary conditions, and let $\mu = \mathcal{N}(0, \Delta^{-\alpha})$, for $\alpha > d/2$.
    Then a sample path $u$ of $\mu$ a.s. belongs to the space $\mathcal{H}^s$ for any $s \in [0,\alpha - d/2).$
\end{lemma}

So, by starting with the Brownian bridge prior, we may enforce additional inductive biases, like continuity or smoothness, by damping the eigenvalues.
Larger values of $\alpha$ will cause the eigenvalues to decay at a faster rate.
This effectively kills off the higher-order frequencies, allowing the sample paths to be differentiable, as understood by Lemma~\ref{lemma:sobolev}.
This is especially useful in Bayesian inverse problems involving partial differential equations, where we have some prior idea about the desired smoothness of the latent functions.

\subsection{Computational cost analysis}

\begin{figure}[h]
    \centering
    \includegraphics[width=0.5\linewidth]{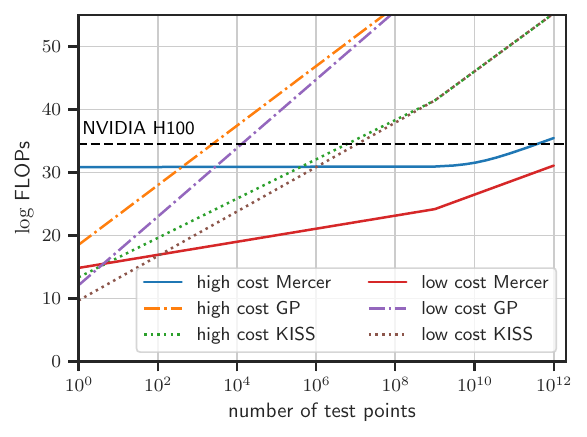}
    \caption{Theoretical cost analysis between the methods.}
    \label{fig:flops}
\end{figure}

We briefly discuss the computational cost of generating a sample from the Mercer prior with one SGLD update using eq.~(\ref{eqn:sampler}).
We let $D$ be the total number of BNN parameters (weights and biases).
For each sampled eigen-index $\alpha = 1,\dots, N$, we need to evaluate two inner summations over $M_1$ and $M_2$ domain points:
$$
\mathcal{S}_{M_1} = \sum_{\beta=1}^{M_1} u_{\theta}(t_{\beta})\phi_{n_{\alpha}}(t_{\beta}),
$$
and similarly for $M_2$.
For simplicity, we will assume that the eigenvalues and eigenfunctions are known analytically and do not need to be numerically computed, e.g., with a Nystr\"om approximation~\cite{williams2000using}.
This holds for the wide class of kernels based on the Laplacian operator commonly used in Bayesian inverse problems, or when engineering a new kernel, both discussed in Sec.~\ref{subsec:kernels}.

Now, each term $u_{\theta}(t)$ requires a forward pass through a neural network, which costs $O(D)$.
Evaluating an eigenfunction $\phi_{n_{\alpha}}(t)$ costs $O(1)$, unless the eigenfunction is a special function, e.g. a Bessel function, which we will ignore.
So, this incurs a cost of $O((M_1 + M_2)D) = O(MD)$ per eigen-index if we take $M_1 = M_2 = M$.
Multiplying by $N$ gives the total leading cost $O(NMD)$.
To be completely fair, we acknowledge the fact that SGLD tends to mix very slowly, so the unbiased estimate needs to be computed many times in order to generate one independent posterior draw.
If we let $T$ denote the number of SGLD steps needed for an independent sample, the total cost is $O(TNMD)$.
The mixing rate can be improved by implementing a variant of SGLD, such as preconditioning~\cite{welling2011bayesian} or cyclical SGLD~\cite{zhang2019cyclical}.
Alternatively, the prior can be characterized by other methods which are known to work with unbiased estimates, such as Hamiltonian Monte Carlo with energy conserving subsampling~\cite{dang2019hamiltonian}, or with variational inference~\cite{blei2017variational}.

To compare the Mercer prior to other methods, we compare the cost in floating point operations per second (FLOPs) of generating a sample from a posterior predictive distribution with a non-Gaussian posterior.
This scenario highlights the Mercer prior's capabilities, especially if the number of test points is large.
To start, let $u \sim \mathcal{GP}(0,k)$ be a Gaussian process, and assume the likelihood involves some nonlinear transformation of $u$ so that the posterior $u | y$ is non-Gaussian, meaning that in all cases sampling is necessary.
Further, we evaluate each resulting sample on a large number of test points, where we will denote the number of test points as $P$.
An example where this situation arises would be in a Bayesian inverse problem with a nonlinear forward map, i.e., the likelihood requires evaluation of a nonlinear PDE.
For simplicity, we will assume that the likelihood costs $O(1)$ per data point.
This is usually not the case, but a likelihood evaluation will cost the same for each method we compare.

If we replace the GP with a BNN equipped with the Mercer prior, generating a sample from the posterior predictive distribution requires the following steps.
First, an independent sample from the BNN posterior is needed, which costs $O(T(NMD + BD))$, where $B$ is the number of minibatch points used in the likelihood (since we sample with SGLD).
Then once an independent sample is identified, we evaluate it on the test points, which costs $O(PD)$.
In total, generating an independent sample from the posterior predictive distribution costs $O(T(NMD + BD) + PD)$.
If we were to use a naive GP, we need to generate a sample from the GP prior on the minibatch points in order to evaluate the likelihood, and then also sample the posterior predictive distribution.
Sampling a GP scales with the cube (for the Cholesky decomposition required to identify the matrix inverse), so this costs in total $O(T(1/3B^3 + B) + 1/3P^3)$, where again, we assume the posterior is sampled with SGLD to exploit minibatching.
For a more scalable GP method, we also consider KISS-GP, which scales linearly in the number of minibatch points and quadratically in the number of test points~\cite{wilson2015kernel}, leading to a cost of $O(2TB + B^2)$.
As there are many hyperparameters to tune, we consider a high cost and low cost scenario, with the values for each reported in Table~\ref{tab:flops}.

In Fig.~\ref{fig:flops}, we compare the computational cost of generating a single sample from the posterior predictive distribution for each method.
We mark where each method crosses the threshold beyond the capabilities of a single NVIDIA H100 GPU, which has a reported computational capability of up to $989$ teraFLOPs.
Even in the high cost case, the BNN with the Mercer prior is able to handle posteriors beyond a billion test points.
In the low cost case, the Mercer prior outperforms the other methods even at the relatively low number of test points close to $100$.

\begin{table}
    \centering
    \begin{tabular}{c c c c c c}
        \text{scenario} & $N$ & $M$ & $D$ & $B$ & $T$ \\
        \hline
        \text{high cost} & $100$ & $10,000$ & $2,500$ & $32$ & $10,000$\\
        \text{low cost} & $10$ & $8$ & $32$ & $8$ & $1,000$
    \end{tabular}
    \caption{Values taken for the cost analysis in the two cases.}
    \label{tab:flops}
\end{table}

\subsection{Proposed directions for hyperparameter selection}
\label{subsec:hyper}
If the covariance operator $S_{\lambda}$ depends on certain kernel hyperparameters $\lambda$, e.g., correlation length or variance, then this dependence carries over into the normalization constant
$$
Z(\lambda) = \int_{\Theta} \exp\left(-E(\theta;\lambda)\right) d\theta, \quad E(\theta;\lambda) \coloneqq \frac{1}{2}\langle u_{\theta}, S^{-1}_{\lambda}u_{\theta}\rangle,
$$
which defines the properly normalized prior $p(\theta|\lambda) = Z(\lambda)^{-1}\exp(-E(\theta;\lambda))$.
To connect the technique back to its inspiration in field theories, $E$ can be viewed as the energy of the system, and $Z(\lambda)$ is the corresponding partition function.
Under this view, the prior assigns higher probability density to neural networks which are closer to the minimum energy state.

A fully Bayesian treatment over both $\theta$ and $\lambda$ requires accounting for the variation of $Z(\lambda)$, as the joint posterior would be
$$
p(\theta,\lambda|y) \propto p(y|\theta)p(\theta|\lambda)p(\lambda) = \frac{1}{Z(\lambda)}p(y|\theta)\exp\left(-E(\theta;\lambda)\right)p(\lambda),
$$
once one picks a hyperprior $p(\lambda)$.
Of course, this is unlikely to be tractable due to the complicated map defined by the neural network architecture within $Z(\lambda)$.
Ignoring this subtlety and dropping the partition function introduces bias into the inference whenever $\lambda$ is learned jointly with $\theta$.
In our examples in this work, we will keep $\lambda$ fixed so that $Z(\lambda)$ is simply a constant and can be dropped when sampling with SGLD.
Our primary concern is how the Mercer prior performs as an approximation to GPs in applications, not necessarily on how any hyperparameters are learned.

If desired, one approach to train any hyperparameters would be to work with the marginal $p(\lambda|y) = \int_{\Theta}p(\theta,\lambda|y)d\theta$, which can be characterized through the following identity
\begin{equation}
    \label{eqn:energy}
    \nabla_{\lambda} \log p(\lambda|y) = \mathbb{E}_{p(\theta|y,\lambda)}[-\nabla_{\lambda}E(\theta;\lambda)] - \mathbb{E}_{p(\theta|\lambda)}[-\nabla_{\lambda}E(\theta;\lambda)] - \nabla_{\lambda}\log p(\lambda),
\end{equation}
which is derived in~\cite{alberts2023physics}.
Let us consider the simplified case where $p(\lambda)$ is flat, so that only the expectations appear.
Then, the practice reduces to model selection by marginal likelihood~\cite{williams2006gaussian}.
This expression states that the gradient of the $\log$-evidence with respect to the hyperparameters is the difference between the expected energy gradients under the posterior and prior, respectively.
Intuitively, the first term dictates how the energy landscape shifts under the posterior, i.e., it points towards the direction of $\lambda$ which reduces the energy.
The second term is a correction coming from the change in normalization constant.
It is a baseline shift in energy that happens even without data.
Again to connect back to physics, the derivative of the energy $\nabla_{\lambda}E(\theta;\lambda)$ can be viewed as a force, so eq.~(\ref{eqn:energy}) describes the expected excess force that $\lambda$ is subjected to induced by the data.
Probabilistically, this is exactly the score function gradient estimator for the likelihood $\nabla_{\lambda}\mathbb{E}_{p(\theta|\lambda)}[p(y|\theta)] = \mathbb{E}_{p(\theta|\lambda)}[p(y|\theta)\nabla_{\lambda}\log p(\theta|\lambda)]$~\cite{williams1992simple}.
This can be seen by first observing that
\begin{align*}
    \nabla_{\lambda}\log p(\theta | \lambda) &= -\nabla_{\lambda}E(\theta;\lambda) - \nabla_{\lambda} \log Z(\lambda) \\
    & = -\nabla_{\lambda}E(\theta;\lambda) + \mathbb{E}_{p(\theta|\lambda)}[\nabla_{\lambda} E(\theta;\lambda)],
\end{align*}
by using the property of the partition function that $\nabla_{\lambda} \log Z(\lambda) = - \mathbb{E}_{p(\theta|\lambda)}[\nabla_{\lambda} E(\theta;\lambda)]$.
Taking the posterior expectation reveals exactly the form of eq.~(\ref{eqn:energy}).

In practice, the identity eq.~(\ref{eqn:energy}) enables empirical Bayes optimization of $\lambda$ using Monte Carlo estimates of the two expectations, which can be obtained via SGLD samples from the posterior and prior with fixed hyperparameters.
A fully Bayesian approach would be possible with nested sampling as found in~\cite{alberts2023physics}.
This approach uses alternating SGLD updates for $\theta$ to approximate the expectations, which then inform an SGLD step for $\lambda$ in order to sample the joint posterior without explicitly computing $Z(\lambda)$.
While exact, methods that require evaluating or estimating $Z(\lambda)$, such as pseudo-marginal Metropolis-Hastings, are typically computationally demanding.
One could also optimize $\lambda$ using the score function gradient estimator, e.g., with Adam~\cite{kingma2014adam}.

\section{Case study: drawing neural networks from Brownian motion}
\label{sec:study}

\begin{figure}[htbp]
    \centering
    \begin{subfigure}[b]{0.475\textwidth}
        \centering
        \includegraphics[width=\textwidth]{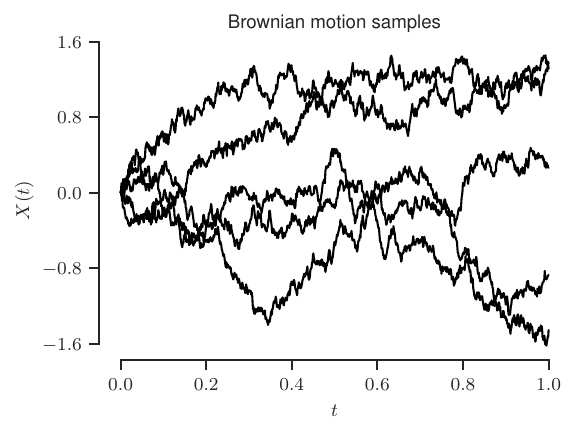}
        \caption{Exact samples}
    \end{subfigure}
    \hfill
    \begin{subfigure}[b]{0.475\textwidth}
        \centering
        \includegraphics[width=\textwidth]{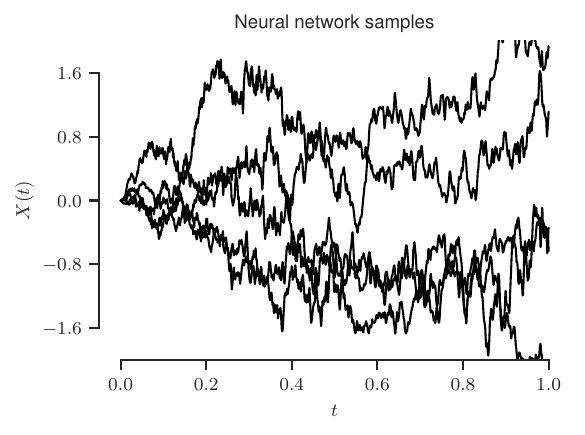}
        \caption{Bayesian neural network with Mercer prior}
    \end{subfigure}
    \caption{Comparison between true Brownian motion and samples generated with the Mercer prior. The samples are generated with $K = 1,000$ eigenvalues in the prior and each sample is evaluated at $100,000$ points.}
    \label{fig:BMsamples}
\end{figure}

To illustrate how the Mercer prior performs in practice, we study how the prior can be use to draw BNNs which follow Brownian motion in great detail.
Let $C_0\left([0,1]\right)$ be the space of real-valued continuous functions $u$ on $[0,1]$ with the property $u(0) = 0$.
We say that a stochastic process follows Brownian motion if it is a sample path from the centered GP with covariance $k(s,t) = \min(s,t)$.
We then define a covariance operator $S$ on $C_0\left([0,1]\right)$ by 
\begin{equation}
    \label{eqn:covop}
    (Su)(t) = \int_0^1 \min(s,t)u(s)ds.
\end{equation}
Then, a Brownian motion $X(t)$ is equivalently a sample path of the Gaussian measure on $C_0\left([0,1]\right)$ with mean $0$ and covariance operator $S$ given  by eq.~(\ref{eqn:covop}).
So, the goal is to draw neural networks which resemble sample paths of the Gaussian measure $\mu \sim \mathcal{N}(0,S)$.
The measure $\mu$ is of course the well-known classical Wiener measure~\cite{kuo2006gaussian}.

On $[0,1]$, the eigenvalues and normalized eigenfunctions of the classical Wiener measure are respectively given by
\begin{align*}
    \lambda_n &= \pi^{-2}(n-0.5)^{-2}, \\
    \phi_n(t) &= \sqrt{2}\sin((n-0.5)\pi t),
\end{align*}
for $n \in \mathbb{N}$.
For the field parameterization, we let $u_{\theta}(t) = tf(t;\theta)$, where $f$ is a single-layer neural network with the sigmoid activation function $\sigma(t) = (1+ e^{-t})^{-1}$.
We use this specific parameterization to automatically satisfy the condition that $u_{\theta}(0) = 0$ in order to follow a Brownian motion.
This ensures that each BNN sample belongs to the space $C_0([0,1])$.
Throughout the experiments, we vary both the width of the network and the number of kernel terms $K$.
In addition to this, we add a Fourier feature input layer to the networks~\cite{tancik2020fourier}.
This is in effort to combat the spectral bias of neural networks and allows the parameterization to more easily capture the characteristic high-frequency nature of Brownian motion.
The inclusion of the Fourier feature layer helps SGLD identify the high-probability region faster.
In principle, this could be omitted, although this would need significantly more samples of SGLD to truly characterize the posterior.
We use the open source package Equinox~\cite{kidger2021equinox}, which is built on top of JAX~\cite{jax2018github}, for implementation, and for sampling with SGLD, we use the samplers found in the Blackjax~\cite{cabezas2024blackjax} library.

Our highest-fidelity approximation is presented in Fig.~\ref{fig:BMsamples}, where we provide evidence that, at least qualitatively, our method produces samples with the same behavior as Brownian motion.
In this case, we truncate the Mercer representation of $k$ to $1,000$ eigenvalues and eigenfunctions.
When evaluating the unbiased estimate of eq.~(\ref{eqn:sampler}) in order to produce an SGLD step, we subsample $100$ eigenvalues and eigenfunctions and use $10,000$ uniformly distributed points in $[0,1]$ for the $L^2$-inner products.
Each SGLD sample produces a different neural network which resembles draws from the classical Wiener measure.
Since the sample paths are represented as neural networks, they can be evaluated on an arbitrarily fine mesh without the cumbersome need to invert a large covariance matrix.
Further, the subsampling capabilities of SGLD allow us to efficiently draw many samples, as the computations can be done in parallel.
For this figure, we generate $100,000$ neural network samples each evaluated at $100,000$ points evenly spaced in $[0,1]$, and select $5$ randomly for plotting.

\begin{figure}[htbp]
    \centering
    \begin{subfigure}[b]{0.475\textwidth}
        \centering
        \includegraphics[width=\textwidth]{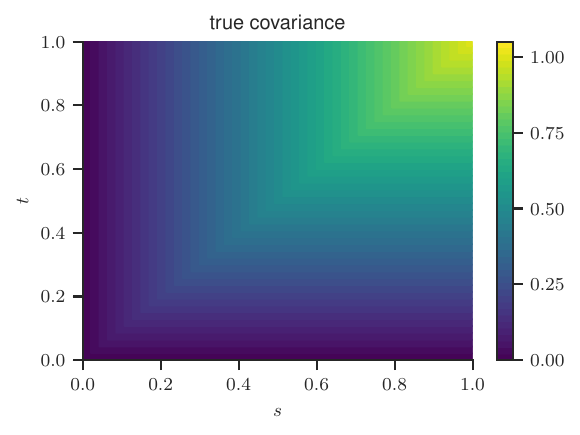}
        \caption{Exact covariance function.}
    \end{subfigure}
    \hfill
    \begin{subfigure}[b]{0.475\textwidth}
        \centering
        \includegraphics[width=\textwidth]{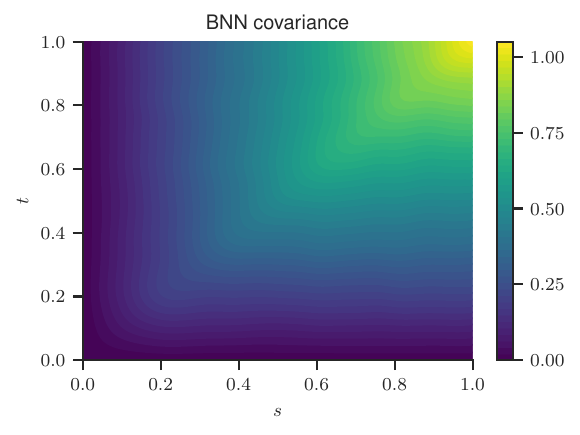}
        \caption{Empirical BNN covariance function.}
    \end{subfigure}
    \caption{Comparison between the true covariance function of Brownian motion $k(s,t) = \min(s,t)$ and the empirical covariance generated with the Mercer prior. The approximation is generated with $K = 1,000$ eigenvalues with a width of $1,000$ neurons.}
    \label{fig:BMcov}
\end{figure}
One way to verify that the BNNs equipped with the Mercer prior are in fact approximately drawn from the Wiener measure is to numerically compare their output statistics to that of true Brownian motion.
Perhaps the most important statistic to check is whether or not the empirical covariance of our BNN samples matches that of Brownian motion.
The empirical covariance resulting from the BNN samples is calculated in the following manner.
For each BNN sample, let $u_{\theta_i}$ be the evaluation of the resulting network on a regular grid of with $m$ points contained in $[0,1]$, for $i=1,\dots, N_F$.
This is produces a random vector in $\mathbb{R}^m$ for each BNN sample $u_{\theta_i}$, so the empirical covariance can be computed in the usual manner.
Denoting by $\bar{u}_{\theta} = \frac{1}{N_F}\sum_{i=1}^{N_F}u_{\theta_i}$ the empirical mean of the ensemble, the empirical covariance matrix $Q$ is given by
$$
Q = \frac{1}{N_F -1}\sum_{i=1}^{N_F}(u_{\theta_i} - \bar{u}_{\theta})(u_{\theta_i} - \bar{u}_{\theta})^{\top}.
$$

In Fig.~\ref{fig:BMcov}, we compare the true covariance function of Brownian motion $k(s,t) = \min(s,t)$ to the empirical covariance matrix of the BNN samples.
When calculating the empirical covariance, we take $N_F = 100,000$ samples, with each network being evaluated on $m =1,000$ points.
We then perform two statistical tests to qualitatively assess the performance of the method.
First, we directly compare the true covariance of Brownian motion to the empirical covariance obtained by the Mercer prior.
The matrix of absolute errors for this case is shown in Fig.~\ref{fig:coverr1000}.
We report a maximum error of $\approx 0.43$, which is located at the tip $s=t=1$.
This represents a discrepancy of less than $5\%$.
Recall that (centered) GPs are uniquely determined by their covariance functions.
This level of agreement between the covariances of true Brownian motion and our BNNs therefore suggests that the BNNs are in fact approximately sampled from the Wiener measure.

In the same manner, we also check that the distribution of BNNs on specific time slices matches true Brownian motion.
At any given time $t$, the possible values of a Brownian motion sample are distributed according to $\mathcal{N}(0,t^2)$.
So, if we collect the output of the BNN samples at a fixed point in time, the result should be a distribution which closely follows this normal distribution.
We present the results of this test for $4$ evenly-spaced points in time between $[0.25,1]$ in Fig.~\ref{fig:brownhist}, which are generated with a kernel density estimation.
It is evident that at each slice of time, the distribution of BNN values matches that closely of true Brownian motion.
While both of these statistical tests are mostly qualitative, they provide strong numerical evidence that our prior works as advertised.
In the following subsections we provide more quantitative evidence.

\begin{figure}[h]
    \centering
    \begin{subfigure}[b]{0.475\textwidth}
        \centering
        \includegraphics[width=\textwidth]{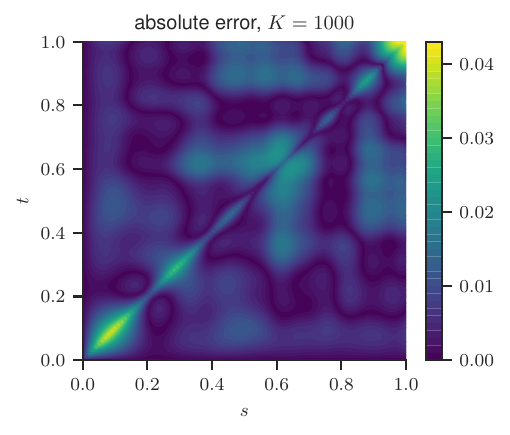}
        \subcaption{Absolute error between the true covariance of Brownian motion and the empirical covariance resulting from the BNN samples drawn from the Mercer prior.}
        \label{fig:coverr1000}
    \end{subfigure}
    \hfill
    \begin{subfigure}[b]{0.475\textwidth}
        \centering
        \includegraphics[width=\textwidth]{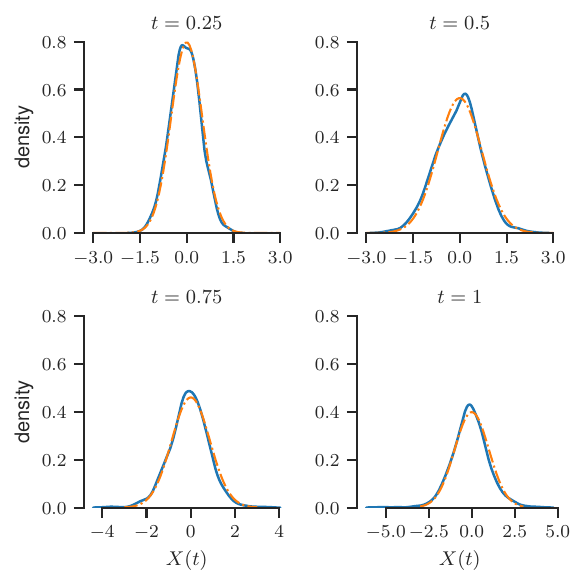}
        \subcaption{Comparison between the distribution of BNN samples on time slices and $\mathcal{N}(0,t^2)$.}
    \label{fig:brownhist}
    \end{subfigure}
    \caption{Statistical tests for BNNs which follow Brownian motion.}
\end{figure}

The main difference between the BNNs and true Brownian motion is that the BNNs still have smoothness to them.
We can observe this in the covariance functions, as the BNN covariance is a bit more blurred.
This can also be seen by the fact that most of the error in the covariance concentrates along the center line $s=t$, which suggests the BNNs have difficultly in capturing the extreme frequencies.
This fact should not be a surprise, as the finite-width neural network will be infinitely differentiable for any set of parameters $\theta$ since we have chosen the sigmoid activation function, while Brownian motion is famously nowhere differentiable.
One may therefore be tempted to replace the architecture with one which produces non-differentiable neural networks.
However, doing so will lead to an ill-defined prior, even for a finite-width neural network.
Recall that before building an unbiased estimate, the BNN prior looks like $p(\theta) \propto \exp\left(-1/2\langle u_\theta, S^{-1}u_\theta\rangle\right)$, where we then express $S^{-1}$ in terms of its Mercer representation.
For Brownian motion, one can show that this prior also looks like
$$
p(\theta) \propto \exp\left(-\frac{1}{2} \int_0^1 \left(\frac{du_{\theta}}{dt}\right)^2dt\right),
$$
which can be found in~\cite{adler2010geometry}.
Hence, for the prior to even be well-defined, the neural network must differentiable so that the density can be computed.
At best we can produce approximations of Brownian motion with at least one derivative which limit to non-differentiable functions.
This sort of dichotomy is not unusual when dealing with Gaussian measures.
As it turns out, this regularity requirement is exactly the statement that the BNN must a priori live in the RKHS of the GP prior, for which the associated Gaussian measure assigns zero probability~\cite{kanagawa2018gaussian}.
Note that this problem disappears if the RKHS is finite-dimensional.

\subsection{Dependence of the prior on the eigenvalues and eigenfunctions}
\label{subsec:evals}

When numerically computing the prior there are two main sources of the error between the BNN samples and the true GP paths.
The first is that we have chosen to truncate the number of eigenvalues and eigenfunctions appearing in the Mercer representation to some finite value $K$.
The second comes from the fact that we cannot use an infinite-width neural network, i.e., for each layer, we use $N_s <\infty$ neurons.
Intuitively, we should expect convergence of the BNN to the desired GP in some limit as both $K$ and $N_s$ go to infinity.
The convergence in $K$ is trivial by Mercer's theorem.
Regardless, we numerically investigate both pieces, starting with the eigenvalues and eigenfunctions.

Within spectral expansion methods, particularly for the KLE, it is conventional wisdom to truncate the expansion at a point which captures a desired `energy' threshold.
That is, we select the truncation value according to a formula such as
$$
\frac{\sum_{n=1}^K \lambda_n}{\sum_{n=1}^\infty \lambda_n} = 0.99.
$$
When using the KLE, this represents an approximation which captures $99\%$ of the variance of the GP.
For Brownian motion, we calculate this as approximately $20$ eigenvalues.

\begin{figure}[htbp]
    \centering
    \begin{subfigure}[b]{0.475\textwidth}
        \centering
        \includegraphics[width=\textwidth]{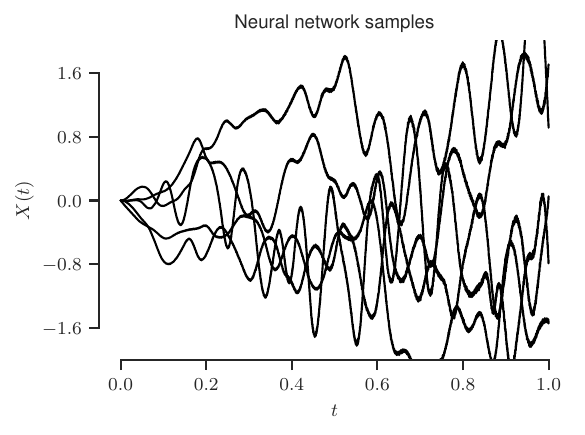}
        \caption{BNN samples generated with $K = 20$.}
    \end{subfigure}
    \hfill
    \begin{subfigure}[b]{0.475\textwidth}
        \centering
        \includegraphics[width=\textwidth]{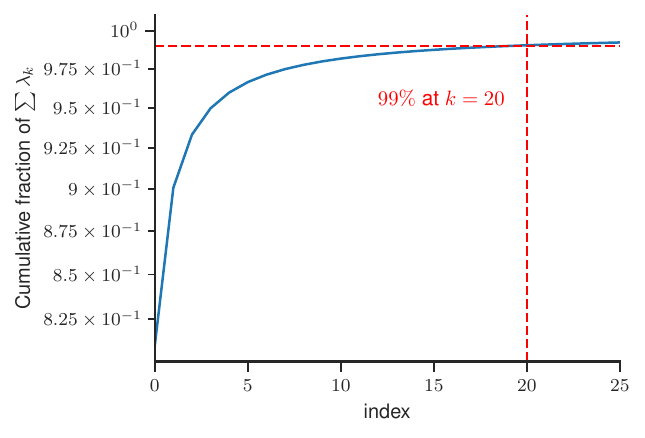}
        \caption{Cumulative ratio of Brownian motion eigenvalues.}
    \end{subfigure}
    \caption{BNNs sampled from the Mercer prior for Brownian motion with $K=20$ eigenvalues and eigenfunctions. On the left in (a), we plot $5$ BNN samples, and on the right in (b), we show the evolution of the cumulatize ratio of eigenvalues.}
    \label{fig:BMsamps20}
\end{figure}

\begin{figure}[htbp]
    \centering
    \includegraphics[width=\linewidth]{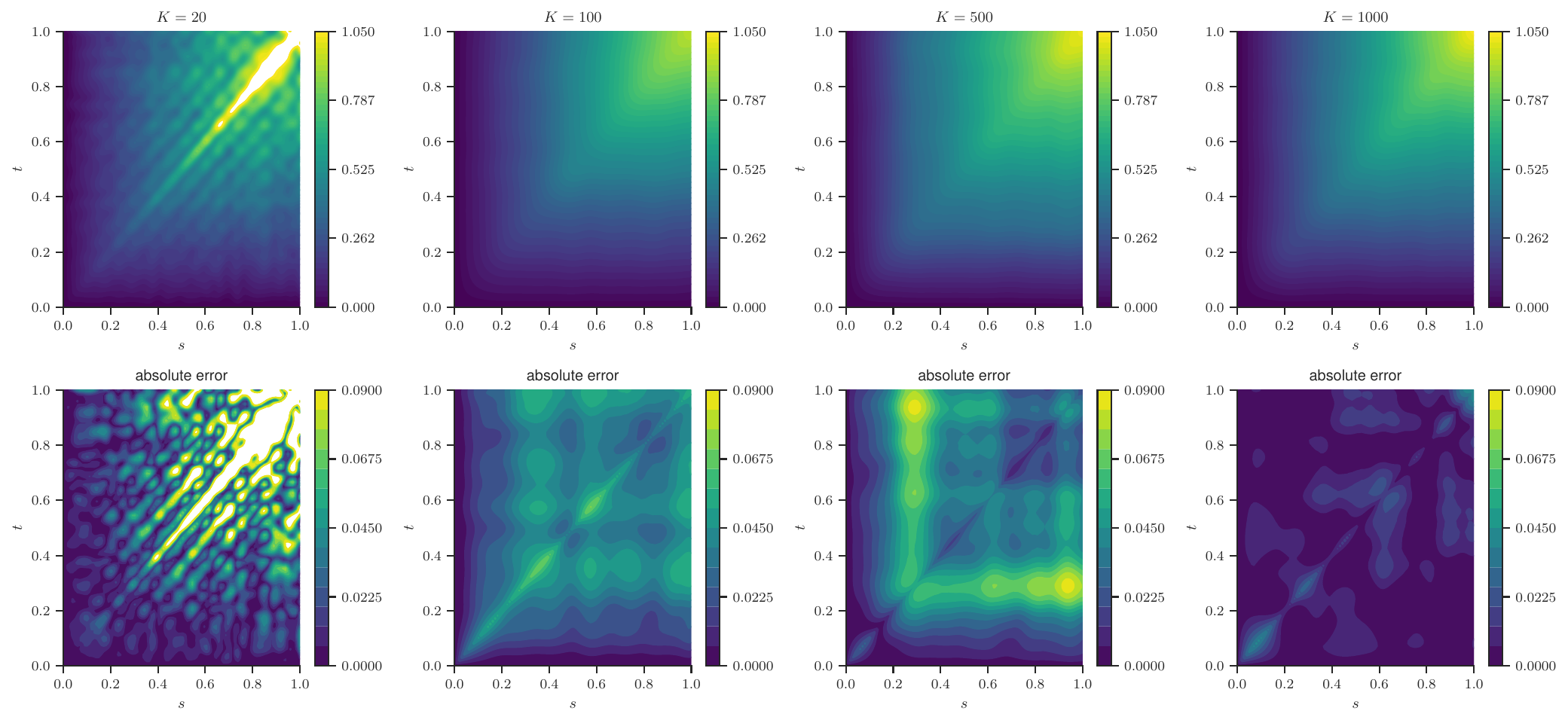}
    \caption{BNN empirical covariance matrices for adjusting values of $K$. The top row shows the empirical covariance matrix for each case, and the bottom row shows the corresponding absolute error between the empirical covariance matrix and the covariance of true Brownian motion.}
    \label{fig:eigentest}
\end{figure}

Following this wisdom, we investigate how the method performs when the expansion is truncated at $20$ terms.
The samples produced in this way are presented in Fig.~\ref{fig:BMsamps20}.
We see that although the samples follow the same overall trend of Brownian motion, the method is unable to capture the finer features.
This is evidence that including the higher-order eigenvalues/eigenfunctions are necessary in order to effectively mimic the higher frequencies of Brownian motion.

\begin{figure}[h]
    \centering
    \begin{subfigure}[b]{0.475\textwidth}
        \centering
        \includegraphics[width=\textwidth]{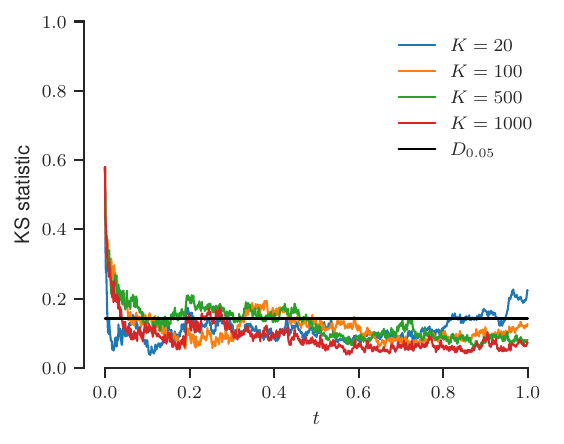}
    \end{subfigure}
    \hfill
    \begin{subfigure}[b]{0.475\textwidth}
        \centering
        \includegraphics[width=\textwidth]{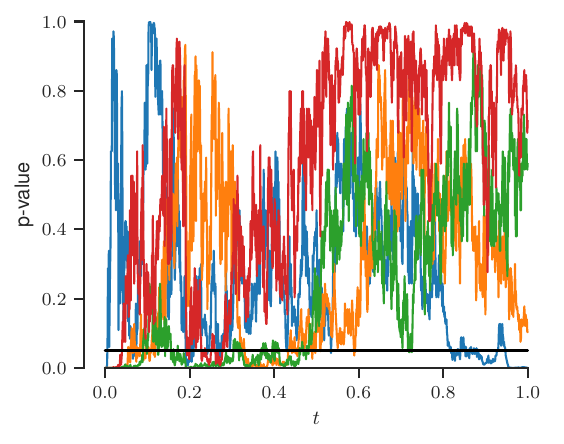}
    \end{subfigure}
    \caption{Results of the two-sample KS test between Brownian motion and the BNN samples. Here, we set the significance level $\alpha = 0.05$.}
    \label{fig:KSevals}
\end{figure}

In Fig.~\ref{fig:eigentest}, we show how the empirical covariance of the BNN samples change according to $K$.
For each case, we use an identical architecture, take the same number of samples, and use the same mesh.
When evaluating the unbiased estimate of $\log p(\theta)$, we also subsample the same number of points in the domain, and subsample $10\%$ of the eigenvalues and corresponding eigenfunctions.
For the small $K$ cases ($K=20$ and $K=100$), the majority of the error is along the centerline $s=t$, with vanishing error as $s$ and $t$ grow further apart.
This shows that the BNNs are capturing the overall trend of Brownian motion, which can also be viewed in Fig.~\ref{fig:BMsamps20}, i.e., the covariance between any two time points $s$ and $t$ agrees with true Brownian motion on a long time scale.
However, without the higher-order terms, the BNN is unable to capture the instantaneous sharp changes, which is why there is significant error along the diagonal.

Finally, to quantitatively measure how close the BNN samples fall to true Brownian motion, we perform a two-sample Kolmogorov-Smirnov (KS) test.
Since the KS statistic is defined for one-dimensional probability distributions, and we are dealing with continuous stochastic processes, we must design the test accordingly.
For this reason, we calculate the KS statistic between $n_1$ samples of Brownian motion and $n_2$ BNN samples at fixed values in time and track the evolution of the KS statistics.
The two are considered to be drawn from the same distribution if the KS statistic falls below the critical value
$$
D_{\alpha} = c(\alpha) = \sqrt{\frac{n_1 + n_2}{n_1n_2}},
$$
where $c(\alpha)$ is a coefficient depending on $0<\alpha<1$, the significance level.
That is, we are testing the two hypotheses
\begin{align*}
    &H_0:\text{ the two samples are drawn from the same distribution} \\
    &H_1: \text{ the two samples are drawn from different distributions,}
\end{align*}
and $\alpha$ determines the $p$-value.
So, as an alternative to $D_{\alpha}$, if we fail to reject the null hypothesis, then we consider the two distributions to be equal.

The results of the two-sample KS test are presented in Fig.~\ref{fig:KSevals}.
For the test, we set the significance level to be $\alpha = 0.05$.
The KS statistic suggests that the marginal distributions of the BNN match that of Brownian motion when approximately $1,000$ eigenvalues/eigenfunctions are accounted for.
We feel uncomfortable making this claim for $K = 500$, as there is a region where the KS statistic largely falls above $D_{0.05}$ (for $t\in(0.2,0.4)$).
The same holds for the $p$-values. 
What is notable is that in all cases, the KS statistic lies outside the acceptable threshold when $t$ is close to $0$.
We attribute this to the specific parameterization we have chosen, i.e., $u_{\theta}(t) = tf(t;\theta)$, which could be forcing the BNN to remain too small close to $0$.
Again, we made this choice so that the initial condition $u_{\theta}(0)=0$ is satisfied to ensure that the operation $S^{-1}u_{\theta}$ is well-defined.

\subsection{Dependence of the prior on the network width}

\begin{figure}
    \centering
    \includegraphics[width=\linewidth]{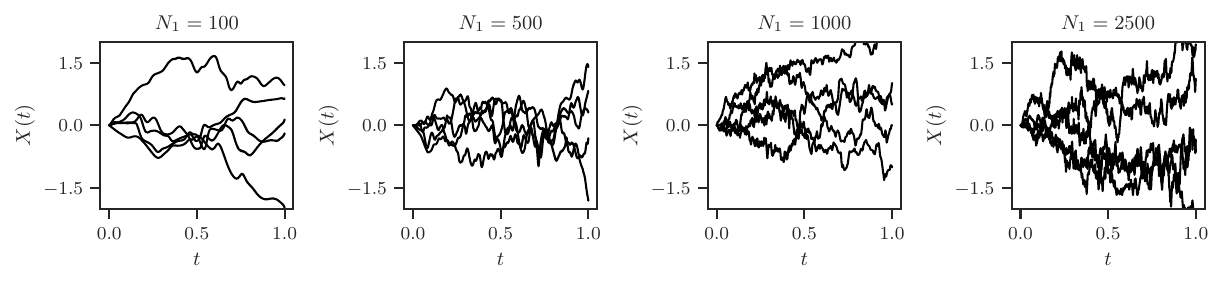}
    \caption{BNN samples from the Brownian motion Mercer prior with varying width.}
    \label{fig:widthtest}
\end{figure}

Next, we numerically investigate how the BNN samples behave as a function of the network width.
For this set of experiments, we select $K=1,000$ eigenvalues and eigenfunctions in the Mercer kernel representation, since this was found to give the best performance as previously studied.
We also use only a single layer neural network, as we find adding additional layers degrades the approximation.

\begin{figure}[h]
    \centering
    \includegraphics[width=\linewidth]{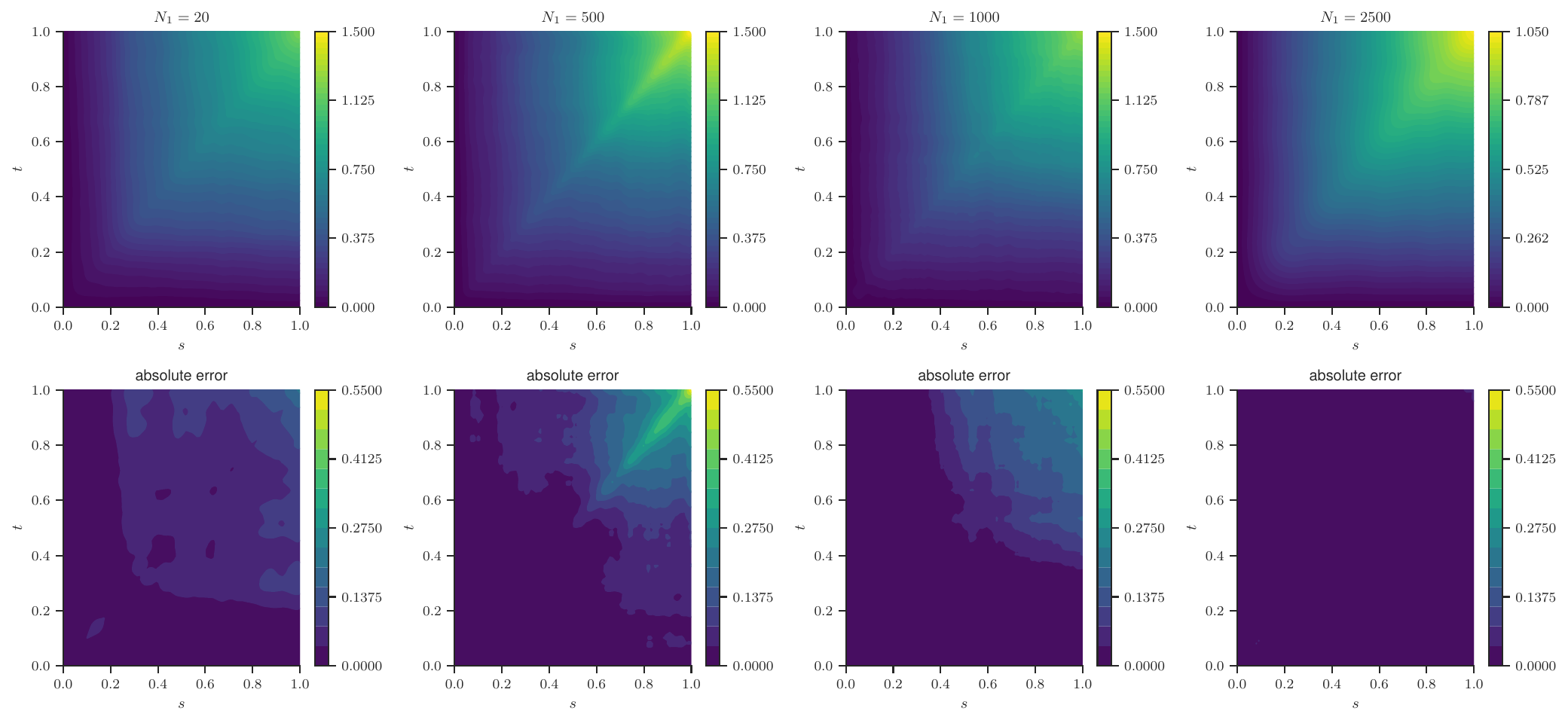}
    \caption{BNN empirical covariance matrices for varying width. The top row shows the empirical covariance matrix for each case, and the bottom row shows the corresponding absolute error between the empirical covariance matrix and the covariance of true Brownian motion.}
    \label{fig:covwidth}
\end{figure}

\begin{figure}[h!]
    \centering
    \begin{subfigure}[b]{0.475\textwidth}
        \centering
        \includegraphics[width=\textwidth]{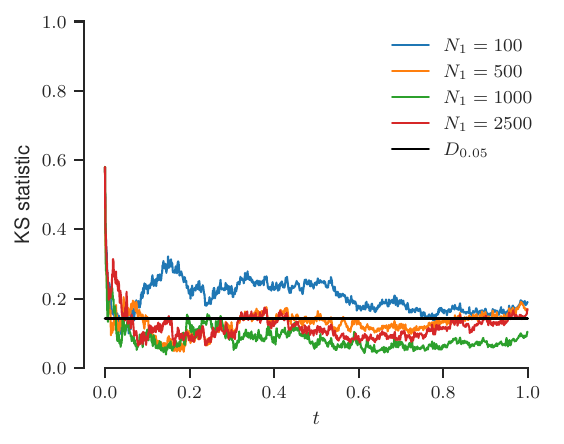}
    \end{subfigure}
    \hfill
    \begin{subfigure}[b]{0.475\textwidth}
        \centering
        \includegraphics[width=\textwidth]{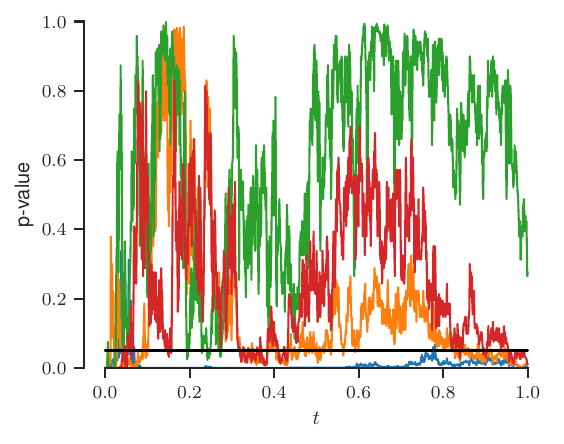}
    \end{subfigure}
    \caption{Results of the two-sample KS test between Brownian motion and the BNN samples for neural networks of different widths. We set the significance level $\alpha = 0.05$.}
    \label{fig:KSwidth}
\end{figure}

In Fig.~\ref{fig:widthtest}, we show a few BNN samples for networks with increasing width, and in Fig.~\ref{fig:covwidth}, we show the corresponding empirical covariance matrices along with the absolute error.
We observe that in general, wider neural networks perform better, which is to be expected.
We also perform a two-sample KS test for each case in the same manner as before, the results of which are presented in Fig.~\ref{fig:KSwidth}.

The results of these numerical experiments show that the width of the neural network plays a key role in how well the BNN prior approximates the target GP.
As the network becomes wider, the approximation becomes higher-fidelity as evidenced in the empirical covariance matrices.
However, the KS test suggests that the improvement plateaus, as both the $1,000$ neuron and $2,500$ neuron network yield similar results.
This suggests diminishing returns once the network is expressive enough relative to the number of eigenfunctions used.
The use of a wider network will reduce the error, but this comes at an increased computational cost.
Together with the results of Sec.~\ref{subsec:evals}, this highlights that both the spectral truncation $K$ and the network width $N_1$ jointly determine the quality of the approximation, with convergence expected in the limit $K,N_1\to\infty$.

\subsection{The Brownian bridge}

To further illustrate how the choice of the covariance operator shapes the behavior of the Mercer prior, we briefly show the method as applied to the Brownian bridge on $[0,1]$.
The Brownian bridge $X(t)$ is defined by conditioning standard Brownian motion with the additional constraint that $X(1) = 0$.
This can be described as the centered GP with covariance kernel $k(s,t) = \min(s,t)-st$.
The eigenvalues of this process are $\lambda_n = (n\pi)^{-2}$ with corresponding eigenfunctions $\phi_n(t) = \sqrt{2}\sin(n\pi t)$.
\begin{figure}[h]
    \centering
    \begin{subfigure}[b]{0.475\textwidth}
        \centering
        \includegraphics[width=\textwidth]{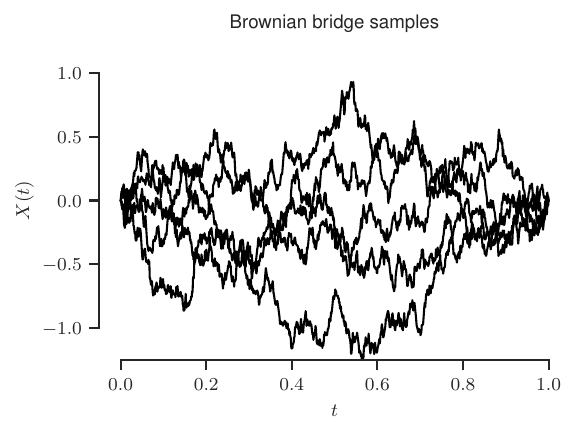}
        \caption{Exact samples from the unit Brownian bridge.}
    \end{subfigure}
    \hfill
    \begin{subfigure}[b]{0.475\textwidth}
        \centering
        \includegraphics[width=\textwidth]{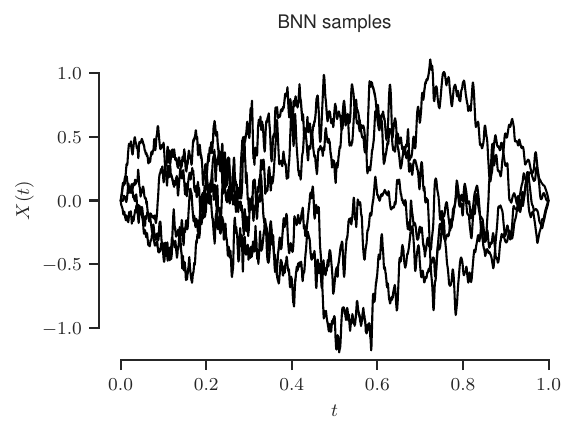}
        \caption{BNN with Brownian bridge Mercer prior.}
    \end{subfigure}
    \caption{Comparison between true Brownian bridge samples and samples from Bayesian neural networks with corresponding Mercer prior.}
    \label{fig:bridgesamps}
\end{figure}
To satisfy the boundary conditions, we take the parameterization to be $X(t;\theta) = t(1-t)u_{\theta}$, where $u_{\theta}$ is a BNN, and $\theta$ is sampled from the Mercer prior.
The same neural network architecture is used as in the high-fidelity Brownian motion case, and we also set $K = 1,000$ eigenvalues.

In Fig.~\ref{fig:bridgesamps}, we show a few samples from a BNN with the Brownian bridge Mercer prior as compared to the true Brownian bridge.
Qualitatively, each realization exhibits the expected features of the Brownian bridge.
A comparison between covariance forms is provided in Fig.~\ref{fig:BBcov}.
While not completely converged, the empirical covariance computed from these samples closely matches the analytical form, showing that the Mercer prior reproduces both the spatial correlation and the suppression of variance near the endpoints of the true GP.
\begin{figure}
    \centering
    \begin{subfigure}[b]{0.32\textwidth}
        \centering
        \includegraphics[width=\textwidth]{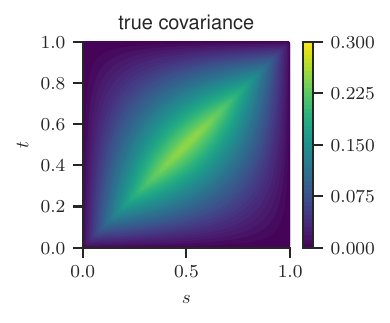}
    \end{subfigure}
    \hfill
        \begin{subfigure}[b]{0.32\textwidth}
        \centering
        \includegraphics[width=\textwidth]{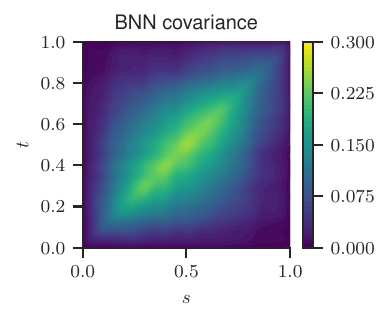}
    \end{subfigure}
    \hfill
        \begin{subfigure}[b]{0.32\textwidth}
        \centering
        \includegraphics[width=\textwidth]{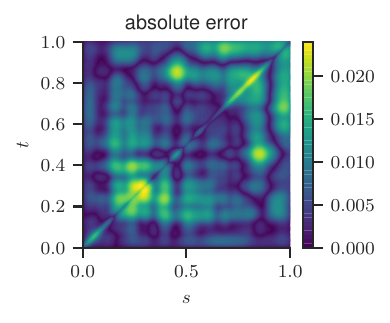}
    \end{subfigure}
    \caption{Comparison between the true covariance of the Brownian bridge process and the empirical covariance matrix resulting from the Mercer prior. In this example, we identify a maximum absolute error of $0.024$.}
    \label{fig:BBcov}
\end{figure}

\section{Example applications}
\label{sec:applications}

We provide some example applications of the Mercer prior in practice.
We highlight situations in which it is beneficial to replace a GP with a BNN, while also showing that the addition of a Mercer prior improves the performance of a BNN when compared to an i.i.d. Gaussian prior.
The examples covered all contain real-world data, and not simple synthetic examples, which we hope helps makes the case that the Mercer prior is worth the effort to implement in practice.
In each example, we report how each prior is chosen, which closely follows the procedure when selecting a GP prior.
We also report the hyperparameters used to create the examples, which are chosen following the usual guidelines.
The Mercer prior adds an additional layer of complexity, where one must decide how many eigenvalues and eigenfunctions to keep, as well as the subsampling sizes used to generate the unbiased estimate.
Compared to the Brownian motion examples, we find much less domain points are required when building an unbiased estimate of the integrals contained in the prior in order to achieve good posterior performance.
We conjecture that this is related to the difference in smoothness of the sample paths.
Following the empirical evidence of Sec.~\ref{subsec:evals}, we settle on keeping the first $100$ eigenvalues and eigenfunctions for kernels based on the Laplacian (for two of the examples), while one of the examples uses an engineered kernel and is discussed later.
In all of our examples, we subsample $10\%$ of the total eigenvalues and corresponding eigenfunctions at each SGLD step.

\subsection{Gaussian process regression with heteroscedastic noise}

For the first example, we highlight the ease with which the networks can be deployed in GP regression tasks with difficult posteriors.
This example demonstrates how the method allows for easy hierarchical GP regression.
We consider the dataset of~\cite{silverman1985some}, which consists of measurements of the acceleration of a helmet during a motorcycle crash at various points in time.
Note that we standardize the measurements and rescale the domain to $[0,1]$.
It is well-known that these individual measurements are subject to heteroscedastic noise.
We build a hierarchical Bayesian model to perform inference on this dataset, replacing any GPs with BNNs sampled from a Mercer prior.

Letting the individual data-pairs be denoted as $(t_i,u_i)_{i=1}^n$, we take the likelihood of an individual observation to be
$$
p(u_i | m,\sigma^2) = \mathcal{N}(u_i|m(t_i),\sigma^2(t_i)),
$$
where $m$ and $\sigma^2$ are functions which represent the evolving mean acceleration and heteroscedastic variance of the measurement, respectively.
We represent both as centered GPs $m \sim \mathcal{GP}(0,\Delta^{-2})$ and $\sigma \sim \mathcal{GP}(0,\Delta^{-2})$, which places the inductive bias that the mean and standard deviation should live in the space $H^1([0,1])$ at the bare minimum, by Lemma~\ref{lemma:sobolev}.
We select a unit variance for both $m$ and $\sigma$ because the data has been standardized.

\begin{figure}[h]
    \centering
    \begin{subfigure}[b]{0.475\textwidth}
        \centering
        \includegraphics[width=\textwidth]{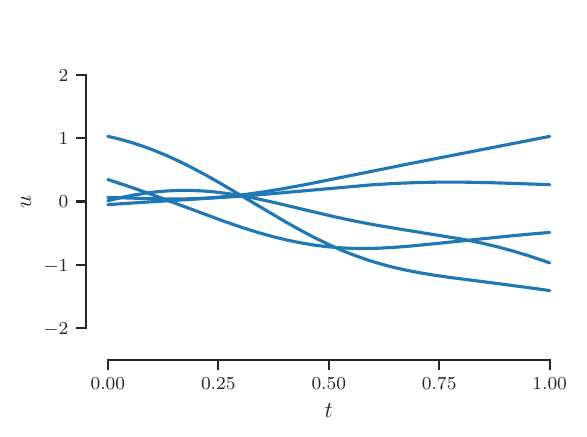}
        \caption{Mercer prior samples for the mean function}
    \end{subfigure}
    \hfill
    \begin{subfigure}[b]{0.475\textwidth}
        \centering
        \includegraphics[width=\textwidth]{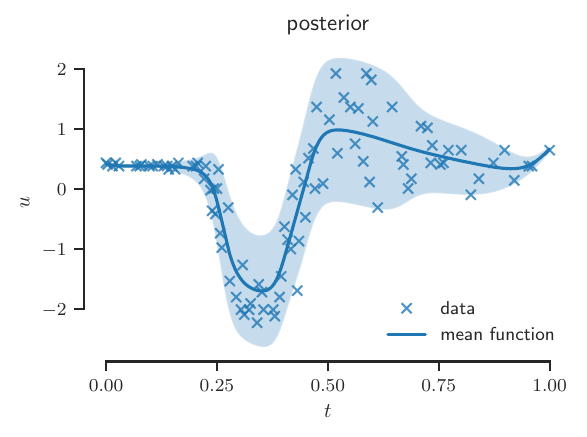}
        \caption{Posterior over the acceleration}
    \end{subfigure}
    \caption{Posterior mean function learned by the BNN Mercer prior model on the motorcycle dataset of~\cite{silverman1985some}. The shaded region represents the $95\%$ predictive interval. The model captures both the overall trend in acceleration and the heteroscedastic noise structure.}
    \label{fig:moto}
\end{figure}

We replace each GP with a BNN equipped with the corresponding Mercer prior, which we denote by $m_{\theta}$ and $\sigma_{\psi}$.
We use the same network architecture for each, which consists of a Fourier feature input layer with $32$ terms followed by a single layer with $1,500$ neurons.
So, each BNN is governed by an individual Mercer prior, i.e., $p(\theta) \propto \exp\left(-1/2 \langle m_{\theta}, \Delta^2 m_{\theta}\rangle \right)$ and $p(\psi) \propto \exp\left(-1/2 \langle \sigma_{\psi}, \Delta^2 \sigma_{\psi}\rangle \right)$.
This gives the hierarchical Bayesian model
\begin{equation*}
    \begin{array}{cc}
         \theta \sim p(\theta)& \psi \sim p(\psi)
    \end{array}
\end{equation*}
$$
u_{1:n} \sim \prod_{i=1}^n\mathcal{N}(m_{\theta}(t_i),\sigma^2_{\psi}(t_i)).
$$
The joint posterior is given by Bayes's rule as
\begin{equation}
    \label{eqn:posterior1}
    p(\theta,\psi | u_{1:n}) \propto p(\theta)p(\psi)\prod_{i=1}^n\mathcal{N}(m_{\theta}(t_i),\sigma^2_{\psi}(t_i)),
\end{equation}
which is then sampled with SGLD.
When sampling, the unbiased estimate is built by keeping the first $100$ eigenpairs of $\Delta^{-1}$.
Then, at each iteration we randomly subsample $10$ eigenpairs, $64$ points in the time domain, and $24$ of the available $94$ datapoints, which represents roughly $25\%$ of the total data.
The initial learning rate is set to $0.001$.
The resulting posterior is shown in Fig.~\ref{fig:moto}.
We see that the BNNs both learn the overall trend of the data and identify the heteroscedastic noise present.
Also, we generate the posterior on a super-fine mesh of one million points to further emphasize the ease at which the method scales.

While the posterior in eq.~(\ref{eqn:posterior1}) resembles a hierarchical GP model, implementing the same hierarchy with a GP would be a bit more difficult.
Sampling the posterior would require an $O(n^3)$ marginal likelihood evaluation at each iteration (to invert the covariance matrix).
By contrast, under the Mercer prior formulation it is straightforward to implement minibatching of the data.
In this way, the Mercer prior provides the interpretability of GP-based hierarchies together with the scalability of neural networks.
While this dataset is relatively small, it is easy to see how this idea extrapolates to datasets which would be computationally prohibitive for GPs.

\subsection{A periodic BNN}

\begin{figure}[h]
    \centering
    \begin{subfigure}[b]{0.3\textwidth}
        \centering
        \includegraphics[width=\textwidth]{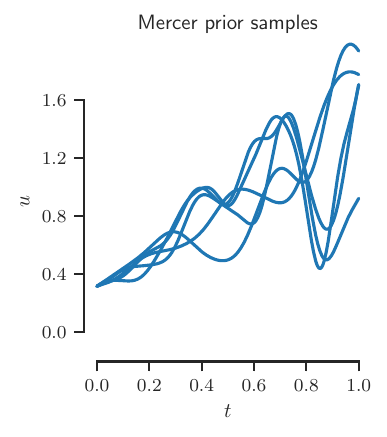}
        \caption{Samples from Mercer prior with periodic eigenfunctions}
    \end{subfigure}
    \hfill
    \begin{subfigure}[b]{0.3\textwidth}
        \centering
        \includegraphics[width=\textwidth]{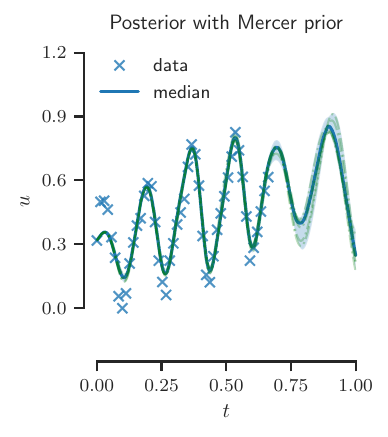}
        \caption{BNN posterior with the Mercer prior}
    \end{subfigure}
    \hfill
    \begin{subfigure}[b]{0.3\textwidth}
        \centering
        \includegraphics[width=\textwidth]{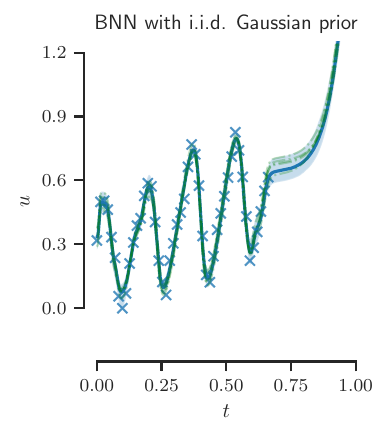}
        \caption{BNN posterior with an i.i.d. Gaussian prior}
    \end{subfigure}
    \caption{Posterior predictive distribution on the $\mathrm{CO}_2$ data for a BNN with a Mercer prior and a BNN with an i.i.d. Gaussian prior. The shaded region represents the $95\%$ credible interval. Samples from the posterior are shown in green.}
    \label{fig:co2}
\end{figure}

For the second example, we repeat an application found in~\cite{matsubara2021ridgelet}, which considers the atmospheric $\mathrm{CO}_2$ concentration levels taken at the Mauna Loa observatory in Hawaii~\cite{keeling2004monthly}.
This is a well-known benchmark dataset for models which encode prior information, due to the strong periodic trend present in the data~\cite{williams2006gaussian}.
We rescale both the measurements and time interval to $[0,1]$ before proceeding.
In this task, we aim to predict the $\mathrm{CO}_2$ concentration in the time interval $t \in (0.65,1]$, where the training data is in the time interval $t \in [0,0.65]$.
In unscaled units, this represents $4$ years of training data, with the predictions up to $2$ years in the future.

We compare the performance of a BNN equipped with a Mercer prior to that of a BNN where the parameters are sampled in the standard manner, i.e., each parameter is sampled an from independent Gaussian.
For the Mercer prior, we look to encode two pieces of prior information, namely that the data follows a linearly increasing trend, and that the data is periodic with a period of approximately $\rho = 0.18$ in scaled time units.
We select a mean function of $m(t) = 0.95t$.
For the covariance, one would typically select the periodic kernel~\cite{mackay1998introduction}, however, the corresponding eigenvalues and eigenfunctions associated with this kernel are not analytically available.
To use the Mercer prior, these would need to be approximated in some manner, e.g., the Nystr\"om approximation.

Instead, we define a custom kernel from predetermined orthonormal functions using the theory of finite-dimensional RKHSs.
We choose the system
$$
\phi_0(t) = \frac{1}{\sqrt{\rho}},\quad \phi_n^{\cos}=\sqrt{\frac{2}{\rho}}\cos\left(\frac{2\pi n t}{\rho}\right), \quad \phi_n^{\sin}= \sqrt{\frac{2}{\rho}}\sin\left(\frac{2\pi n t}{\rho}\right), \: 1\leq n\leq 5,
$$
which are orthonormal on $[0,\rho]$, and are then extended periodically to the full time interval.
For the corresponding eigenvalues, we choose $\lambda_0 = 1$ and for $n\neq 0$, we pick $\lambda_n = \exp(-2 (n\pi/\rho)^2)$, with algebraic multiplicity two, meaning that for each $n$, the corresponding $\phi_n^{\cos}$ and $\phi_n^{\sin}$ are paired with the same eigenvalue.
The eigenfunctions are chosen so that at the bare minimum, the kernel is able to capture the overall periodic trend of the data.
That is, $n=1$ is associated with eigenfunctions with the same period as the data.
The eigenvalues decay very rapidly (hence why we use so few terms), which discourages the BNN from learning higher frequency fluctuations.

We use different neural network architectures for each prior to keep the example fair.
The BNN equipped with the Mercer prior follows our typical case: a Fourier feature input layer with $32$ terms, followed by a single layer with $1,500$ neurons.
For the i.i.d. Gaussian BNN, we use a more standard architecture since wide, shallow neural networks, which we find works well for the Mercer prior, are uncommon in traditional machine learning.
Therefore, we select a deep neural network with $3$ hidden layers and $64$ neurons in each layer.
In both cases, we push the network parameters to the maximum likelihood estimate (MLE) before sampling with SGLD.
This helps decrease the amount of burn-in steps needed for SGLD (recall that the first phase of SGLD acts as a stochastic optimizer before generating samples).
The MLE is identified with $10,000$ steps of Adam with an initial learning rate of $0.01$.
For sampling/training with Adam, we select a minibatch size of $16$ of the $48$ total datapoints.
In this example, we subsample $24$ points in time in the domain and $5$ eigenpairs when building the estimate required for the Mercer prior.
The SGLD initial learning rate is $1\times 10^{-5}$, and we take $100,000$ posterior samples for each BNN.

Figure~\ref{fig:co2} shows the posterior predictive distribution and credible intervals for a BNN with the resulting Mercer prior and a BNN with i.i.d. Gaussian parameters.
The results show how the Mercer prior is able to enforce non-trivial prior information into the BNN, as the dominant periodic behavior is maintained in the future predictions.
Also of note is that the uncertainty grows appropriately in the region without data.
This example highlights how the eigenvalues and eigenfunctions can be tailored for a Mercer prior to handle specific problems.

\subsection{Nonlinear Bayesian inverse problem}

For the final example, we demonstrate how a BNN with a Mercer prior can serve as a replacement for a Gaussian measure in a nonlinear Bayesian inverse problem.
This example is related to the design of spacecraft thermal protection systems (TPS), where it is critical to understand the insulation properties of certain materials so that a TPS may be appropriately sized.
Real world examples include the Apollo spacecraft~\cite{pavlosky1974apollo} or the Space Shuttle~\cite{curry1993space}.
For this example, we consider the dataset found in~\cite{daryabeigi2024thermal}, where the thermal conductivity $\kappa$ of the Opacified Fibrous Insulation (OFI) is identified through experiments.
The dataset consists of steady-state temperature measurements of the OFI ranging from $300$ K to $1900$ K.
In total, we have $51$ temperature measurements of the material.
The goal is to identify the thermal conductivity of the OFI from the temperature measurements with uncertainty, which we then compare to the published experimental values.

This inverse problem is governed by the nonlinear steady state heat equation
\begin{equation}
    \label{eqn:heat}
    \frac{\partial}{\partial y}\left(\kappa(T)\frac{\partial T}{\partial y}\right)=0,
\end{equation}
where the domain is scaled from $0\leq y\leq 1$, $T$ denotes the temperature, and $\kappa$ denotes the thermal conductivity of the OFI, which is known to depend on the temperature.
The boundary conditions are $T(0) = 300$ and $T(1) = 1900$.
The resulting temperature measurements are denoted by $\mathbf{d} = (T_i)_{i=1}^{51}$ and are assumed to be corrupted with zero-mean Gaussian noise with a standard deviation of $\sigma = 2.5$ K.
This results in an infinite-dimensional nonlinear inverse problem.

Under the classical approach, one would start by assigning a Gaussian measure to the thermal conductivity $\mu_0 = \mathcal{N}(m,S)$, with $\log \kappa \sim \mu_0$ so that Bayes's theorem can be used to derive the posterior.
The measurement operator $R$ is given by the solution to eq.~(\ref{eqn:heat}) at the measurement locations.
Therefore the posterior $\mu^{\mathbf{d}}$ is given by the Radon-Nikodym derivative
$$
\frac{d\mu^{\mathbf{d}}}{d\mu_0}(\kappa) \propto \exp\left(-\frac{1}{2}\|\sigma^{-2}I(\mathbf{d} - R(\kappa))\|^2\right).
$$
As $R$ is nonlinear, the resulting posterior is not analytically tractable, and a sampling method is needed.
This presents a computational bottleneck, as repeated inversion of a large covariance matrix is needed to generate each sample.
Potential applications of this problem include simulation of atmospheric re-entry scenarios~\cite{walker2015multifunctional, edquist2009aerothermodynamic, wright2014sizing}, where a very fine mesh would be needed to accurately predict the temperature on the surface of a spacecraft.

We demonstrate how this inverse problem can be solved by replacing the prior Gaussian measure with a BNN with the equivalent Mercer prior.
For the prior knowledge, we need to impose the condition that $\kappa$ has at least one derivative, as $\partial \kappa / \partial T$ appears in eq.~(\ref{eqn:heat}).
Rather than defining the prior in the $\log$-space, we instead apply a softplus transformation, $\mathrm{softplus}(x) = \ln(1+e^x)$, which helps prevent underflow.
Putting this together, we select the prior $\mathrm{softplus}(\kappa) \sim \mathcal{N}(0,\Delta^{-2})$.
We represent this prior as a BNN $\kappa_{\theta}$ with corresponding Mercer prior, with the eigenvalues and eigenfunctions being as discussed in Sec.~\ref{subsec:kernels}.
The BNN architecture in this example is a simple one layer network with $32$ neurons, where again we keep the first $100$ eigenpairs in the expansion.
We find that the BNN does not need to be very expressive in order to produce prior samples which could reasonably represent $\kappa$, as evidenced in Fig.~\ref{fig:bipprior}.

\begin{figure}
    \centering
    \begin{subfigure}[b]{0.475\textwidth}
        \centering
        \includegraphics[width=\textwidth]{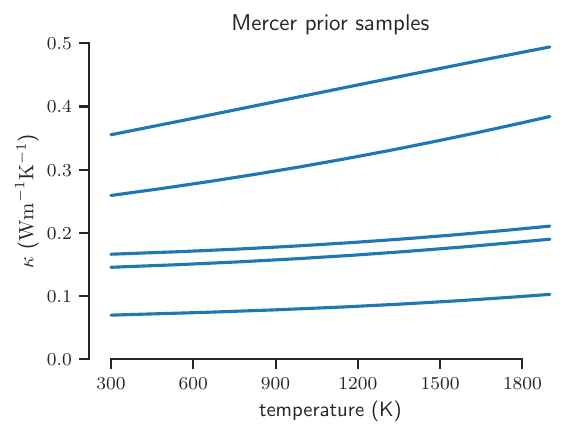}
        \caption{Prior thermal conductivity samples}
        \label{fig:bipprior}
    \end{subfigure}
    \hfill
    \begin{subfigure}[b]{0.475\textwidth}
        \centering
        \includegraphics[width=\textwidth]{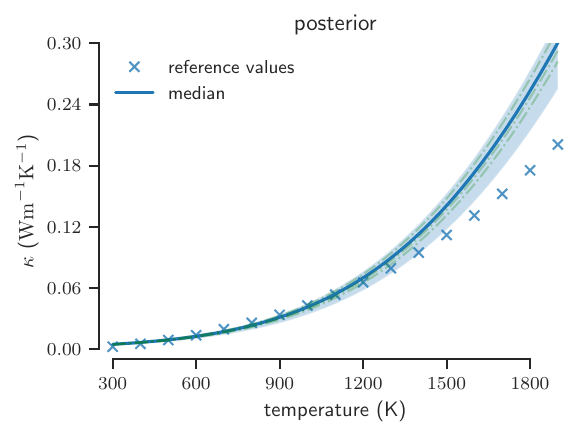}
        \caption{Posterior thermal conductivity samples}
    \end{subfigure}
    \caption{Posterior thermal conductivity of the OFI. The shaded region represents the $95\%$ credible interval, and a few posterior samples are shown in green.}
    \label{fig:bip}
\end{figure}

As the prior is now finite-dimensional, the posterior can be summarized by Bayes's rule as
$$
p(\theta | \mathbf{d}) \propto \exp\left(-\frac{1}{2}\|\sigma^{-2}I(\mathbf{d} - R(\kappa_{\theta}))\|^2\right) p(\theta).
$$
Before sampling from the posterior, we initialize the parameters at the maximum a posteriori estimate $\theta_0 = \arg \max p(\theta|\mathbf{d})$.
Each SGLD step uses $16$ randomly subsampled points in the domain to compute the $L^2$-inner products, $10$ randomly sampled eigenpairs, and a minibatch of $17$ temperature measurements, which is roughly $1/3$ of the data.
We take $1$ million SGLD steps to characterize the posterior.
The resulting BNN posterior is presented in Fig.~\ref{fig:bip}, where we find that the posterior accurately reconstructs the latent thermal conductivity, although there is some discrepancy at higher temperatures.
As this is an example with real experimental data, some amount of error is to be expected.

These results demonstrate that BNNs equipped with Mercer priors are well suited for Bayesian inverse problems governed by PDEs.
In particular, this approach has advantages over the standard GP-based formulation when the forward model is nonlinear, as posterior computation remains scalable thanks to SGLD.
The Mercer prior therefore provides a practical and interpretable route to uncertainty quantification in PDE-constrained inverse problems that would otherwise be computationally prohibitive.

\section{Conclusions}

In this work, we introduced Mercer priors, a new class of priors for BNNs derived from the Mercer representation of a covariance kernel.
The prior enforces a covariance structure on the parameters of the BNN such that in the output space the networks behave similarly to the corresponding GP with the chosen covariance.
This allows for BNNs to inherit the interpretability of GP priors while retaining the scalability and flexibility of neural networks.
One of the main advantages of the method is it works well for neural networks with relatively simple architectures and does not require much in the way of feature engineering.
When it comes to selecting a model, the majority of the effort is spent in choosing an appropriate GP prior that we wish to approximate.
A challenge that remains is the treatment of any hyperparameters appearing in the GP covariance, which was discussed in detail in Sec.~\ref{subsec:hyper}.

Through a series of examples, we investigated the performance of the Mercer prior both in simulating GPs and in applications.
By studying how the prior approximates Brownian motion, which has well-known analytical properties, we investigated how the truncation of the spectral expansion and the neural network width affected the fidelity of the method.
This was done with various statistical tests, which provided a great deal of empirical evidence that convergence is expected in the infinite-width limit.
However, rigorous convergence results connecting the Mercer prior to the corresponding Gaussian measure remains an important theoretical challenge.
In the hierarchical regression and periodic modeling examples, we showed how the Mercer prior allows the BNN to retain the properties of a GP while allowing for efficient and scalable posterior computation.
Then, in nonlinear PDE inverse problem, we demonstrated how the Mercer prior can serve as a replacement in applications which remain computationally prohibitive for the usual Gaussian prior.

Overall, this work highlights how meaningful priors for BNNs can be obtained by changing the parameter distribution, rather than the network structure.
By embedding a Gaussian covariance form directly into the neural network prior, we obtain models that are interpretable and computationally scalable, opening new possibilities for uncertainty quantification, inverse problems, and scientific machine learning.

\paragraph{Acknowledgments}
We would like to thank Kamran Daryabeigi and Akshay Jacob Thomas for providing the data used in example 5.3.
Thanks also to Sascha Ranftl and Jan Fugh for the conversation which ultimately led to these ideas, and thanks to Emmanuel Fleurantin for encouraging us to put these ideas to paper.
\paragraph{Funding}{This work was supported by an NSF Computational and Data-Enabled Science and Engineering grant under award number \#2347472.}

\bibliographystyle{plain}
\bibliography{references}

\end{document}